\documentclass{article}

\usepackage[final]{staix_2026}    % Camera-ready, Full Paper Track
\usepackage{microtype}
\usepackage{graphicx} % For including graphics
\graphicspath{{../figures/}}
\usepackage{subcaption} % For subfigures
\usepackage{booktabs} % For professional looking tables
\usepackage{amsmath} % For math environments
\usepackage{amssymb} % For additional math symbols
\usepackage{mathtools} % For additional math tools
\usepackage{amsthm} % For theorem environments
\usepackage[textsize=tiny]{todonotes} % For inline TODO notes
\usepackage{hyperref} % For hyperlinks
\usepackage[capitalize, noabbrev]{cleveref} % For clever referencing
\usepackage[most]{tcolorbox}

\newtcolorbox{theorembox}{
    colback=blue!5,
    colframe=blue!60!black,
    boxrule=0.8pt,
    arc=2mm,
    left=2mm,
    right=2mm,
    top=1mm,
    bottom=1mm
}

\newcommand{\E}{\mathbb{E}}

\newcommand{\R}{\mathbb{R}}
\newcommand{\Prob}{\mathbb{P}}
\newcommand{\iid}{\overset{\text{i.i.d.}}{\sim}}
\newcommand{\ind}{\overset{\text{ind.}}{\sim}}
\newcommand{\indprop}{\overset{\text{ind.}}{\propto}}

\DeclareMathOperator*{\argmin}{argmin}

\theoremstyle{plain}
\newtheorem{theorem}{Theorem}[section]
\newtheorem{proposition}[theorem]{Proposition}
\newtheorem{lemma}[theorem]{Lemma}

\theoremstyle{definition}
\newtheorem{definition}[theorem]{Definition}

\theoremstyle{remark}

\usepackage{float}
\usepackage{multirow}
\usepackage{siunitx}
\usepackage{adjustbox}

\title{Density-Informed Pseudo-Counts for Calibrated Evidential Deep Learning}

\author{%
  Pietro Carlotti\thanks{Equal contribution.}\\
  Department of Statistics and Data Sciences \\
  University of Texas at Austin \\
  \texttt{pietro.carlotti@utexas.edu}
  \And
  Nevena Gligić$^*$ \\
  Department of Statistics and Data Sciences \\
  University of Texas at Austin \\
  \texttt{nevena.gligic@utexas.edu}
  \And
  Arya Farahi \\
  Department of Statistics and Data Sciences \\
  University of Texas at Austin \\
  \texttt{arya.farahi@austin.utexas.edu}
}

\begin{document}

\maketitle

% {\small $^*$ Equal contribution.} % Uncomment for camera-ready ([final] mode)

% Abstract
\begin{abstract}
    Evidential Deep Learning (EDL) is a popular framework for uncertainty-aware classification that models predictive uncertainty via Dirichlet distributions parameterized by neural networks. Despite its popularity, its theoretical foundations and behavior under distributional shift remain poorly understood. In this work, we provide a principled statistical interpretation by proving that EDL training corresponds to amortized variational inference in a hierarchical Bayesian model with a tempered pseudo-likelihood. This perspective reveals a major drawback: standard EDL conflates epistemic and aleatoric uncertainty, leading to systematic overconfidence on out-of-distribution inputs. To address this, we introduce Density-Informed Pseudo-count EDL, a new parametrization that decouples class prediction from uncertainty quantification by separately estimating the conditional label distribution and the marginal covariate density. This separation preserves evidence in high-density regions while shrinking predictions toward a uniform prior for out-of-distribution data. Theoretically, we prove that our method achieves asymptotic concentration. Empirically, we show our method enhances interpretability and improves robustness and uncertainty calibration under distributional shift.
\end{abstract}

\section{Introduction}
\label{sec:introduction}

Uncertainty quantification is a central challenge in modern machine learning, particularly in safety-critical and high-stakes applications such as medical diagnosis \citep{begoli2019need} and scientific discovery \citep{gal2022bayesian}. While deep neural networks achieve remarkable predictive accuracy, they are notoriously prone to overconfidence \citep{guo2017calibration}, especially under limited data or distributional shift \citep{seligmann2023beyond,mucsanyi2024benchmarking}. Reliable uncertainty estimates are therefore essential for calibrated decision-making, detection of out-of-distribution (OOD) inputs, robust deployment, and supporting downstream reasoning. Despite its importance, principled, scalable uncertainty quantification for deep models remains an open problem.

A wide range of methods has been proposed to address this challenge \citep{mucsanyi2024benchmarking,he2025survey}. Bayesian neural networks aim to capture epistemic uncertainty by placing priors over model parameters and performing posterior inference, typically via variational approximations or Monte Carlo sampling \citep[][]{blundell2015weight,hernandez2015probabilistic,gal2016dropout}. Ensemble-based methods approximate posterior uncertainty by aggregating predictions from multiple trained models, yielding strong empirical performance at the cost of increased computational burden \citep[][]{lakshminarayanan2017simple,wang2023diversity}. Alternatively, post-hoc calibration methods adjust predictive confidence without explicitly modeling uncertainty \citep[][]{guo2017calibration,gibbs2025conformal,vashistha2025calibration}. While effective in certain problems, these approaches often suffer from scalability, complex training pipelines, or limited interpretability of uncertainty estimates.

Evidential Deep Learning (EDL) offers an alternative that directly models predictive uncertainty at the output level \citep{sensoy2018evidential}. Rather than placing distributions over network parameters, EDL predicts a distribution over class probabilities by parameterizing a Dirichlet distribution with a neural network. This formulation provides a unified representation of predictive confidence and uncertainty, allows for single forward-pass inference, and integrates naturally with standard deep learning pipelines. Hence, EDL has gained popularity as a practical framework for uncertainty-aware classification \citep{gao2025comprehensive}. The main idea of EDL is to map the covariates $X \in \R^{d}$ to a distribution over the class probabilities $p \in \Delta^{K-1}$ for a categorical response variable $Y \in \{1, \ldots, K\}$. The original formulation of \citet{sensoy2018evidential} achieve this by specifying the distribution of the class probabilities using a Dirichlet distribution with concentration parameters given by a neural network $\text{NN}^{\phi}: \R^{d} \to \R_{+}^{K}$ with parameters $\phi$ which outputs a $K$-dimensional vector of non-negative weights.

To train this model, \citet{sensoy2018evidential} propose a class of loss functions with a structure 
\begin{equation}
\label{eq:EDL_loss_function}
    \mathcal{L}^{\lambda}_{\text{EDL}}(\phi) = \mathcal{L}_{\text{data}}(\phi) + \lambda \, \mathcal{L}_{\text{reg}}(\phi), 
\end{equation}
where $\mathcal{L}_{\text{data}}$ encourages the model to fit the training data, $\mathcal{L}_{\text{reg}}$ penalizes complexity, and $\lambda > 0$ is a hyperparameter that balances the trade-off between the two. In particular, we will focus on the specific choice of data and regularization terms proposed by \citet{sensoy2018evidential}:
\begin{equation}
    \begin{aligned}
        \mathcal{L}_{\text{data}}(\phi) & = \sum_{i=1}^{n} - \E_{\text{Dir} \left( \alpha + \text{NN}^{\phi}_{X_{i}} \right)} \left[ \log \text{Cat}(Y_{i} \mid p_i) \right], \\
        \mathcal{L}_{\text{reg}}(\phi) & = \sum_{i=1}^{n} \text{KL}\left( \text{Dir} \left( \alpha + \text{NN}^{\phi}_{X_{i}} \right) || \, \; \text{Dir}(\alpha) \right),
    \end{aligned}
\end{equation}
where $\alpha \in \R_{+}^{K}$ is prior concentration parameters. Despite its intuitive appeal and empirical success, the theoretical foundations of EDL and its behavior under distributional shift remain poorly understood. In particular, it is unclear how the learned Dirichlet distributions relate to the underlying data-generating process (DGP), and whether EDL can meaningfully distinguish among different sources of uncertainty. See Appendix~\ref{sec:appendix_related_literature} for a more detailed review of related work on EDL and its limitations.
%, and $\text{Cat}$ and $\text{Dir}$ denote the categorical and Dirichlet probability mass/density functions, respectively.

In this work, we provide a statistical interpretation of the EDL framework, elucidate the role of the regularization parameter, and propose a novel parametrization, Density-Informed Pseudo-count EDL (DIP-EDL), with the goal of improving performance for both in-distribution and out-of-distribution settings. DIP-EDL is a principled method that decouples class prediction from uncertainty quantification by separately estimating the conditional label distribution and the marginal covariate density. %Section~\ref{sec:statistical_interpretation} frames EDL as amortized variational inference with a tempered pseudo-likelihood, highlighting how performance depends on the regularization parameter. Section~\ref{sec:choice_temperature_parameter} introduces DIP-EDL, a principled method that decouples class prediction from uncertainty magnitude by separately estimating conditional label distributions and marginal covariate densities. %This mechanism preserves evidence in high-density regions while shrinking predictions toward a uniform prior for OOD data.
%Section~\ref{sec:experiments} evaluates our method on synthetic and real-world datasets. Section~\ref{sec:conclusion} concludes.

\section{Statistical Interpretation of EDL}
\label{sec:statistical_interpretation}

Here, we present a unifying interpretation of EDL: (i) a hierarchical Bayesian model with amortized variational inference (VI), and (ii) an empirical risk minimization (ERM) approach. We show that both formulations yield the same optimization objective, thereby providing a statistical interpretation and clarifying the uncertainty quantification properties of EDL methods.

\subsection{Problem Setup}
\label{sec:problem_setup}

Suppose the observed data is generated according to the DGP
\begin{equation}
    (Y_{i}, X_i) \iid P^{*}_{Y,X}, \qquad i = 1, \ldots, n.
    \label{eq:DGP}
\end{equation}
The joint distribution admits the factorization $P^{*}_{Y,X} = P^{*}_{Y \mid X}\, P^{*}_{X}$, where $P^{*}_{Y \mid X}$ denotes the conditional distribution of $Y$ given $X$, and $P^{*}_{X}$ the marginal distribution of $X$.
Throughout the analysis, we assume that $P^{*}_{X}$ is absolutely continuous with respect to the Lebesgue measure on $\R^{d}$. Our goal is to estimate $P^{*}_{Y \mid X}$ and provide calibrated uncertainty quantification for in- and out-of-sample predictions.

\subsection{Amortized VI}
\label{sec:amortized_vi}

We show that EDL admits an interpretation as amortized VI on a modification of the Independent Categorical-Dirichlet (ICD) model, in which class probabilities vary across observations.

\begin{definition}
\label{def:icd_model}
    \textbf{ICD Model.} For some $\alpha \in \R_{+}^{K}$, the model is defined as follows:
    \begin{align*}
        Y_i \mid p_i &\ind f_{p_{i}} = \text{Cat}(p_i), \quad i = 1, \ldots, n, \\
        p_i &\iid \pi = \text{Dir}(\alpha), \quad i = 1, \ldots, n.
    \end{align*}
    
\end{definition}

This prior is convenient because its conjugacy with the categorical likelihood allows the posterior of class probabilities to be computed in closed form, as summarized in the following proposition.

\begin{proposition}
\label{prop:icd_model_posterior}
    Given Definition~\ref{def:icd_model}, the posterior and posterior predictive distributions are
    \begin{align*}
        \pi_{Y_{i}} = \mathrm{Dir}(\alpha+e_{Y_{i}}), \; i = 1, \ldots, n, \quad p_{1:n}\!\mid\!Y_{1:n} \sim \pi_{Y_{1:n}} = \prod_{i=1}^{n} \pi_{Y_{i}}, \quad Y_{n+1}\!\mid\!Y_{1:n} \sim \mathrm{Cat}\left(\frac{\alpha}{\alpha_{0}}\right),
    \end{align*}
    where $e_{Y_{i}} \in \{0,1\}^{K}$ denotes the one-hot encoding of $Y_{i}$, $i = 1, \ldots, n$, and $\alpha_{0} = \sum_{k=1}^{K} \alpha_{k}$.
\end{proposition}

Despite its analytical tractability, this model neglects covariates entirely. While a covariate-indexed extension can incorporate them (see Appendix~\ref{proof:covariate_indexed_cd_model_posterior}), its applicability is limited to discrete inputs. Because $P_X^*$ is continuous, each observation is almost surely unique, preventing information sharing. To address this, we introduce covariate dependence via amortized VI \citep{margossian2023amortized}.

\begin{definition}
\label{def:amortized_vi_icd}
    \textbf{Amortized VI for the ICD Model.} Given Definition~\ref{def:icd_model}, we approximate the posterior via a mean-field variational family:
    \begin{align*}
        p_{1:n} \mid Y_{1:n} \approx q^{\widehat{\phi}_n}_{X_{1:n}} = \argmin_{q^{\phi}_{X_{1:n}}} \left\{ \text{KL}\left( q^{\phi}_{X_{1:n}} \,\|\, \pi_{Y_{1:n}} \right) \right\},
    \end{align*}
    where
    \begin{align*}
        q^{\phi}_{X_{1:n}} = \prod_{i=1}^{n} q^{\phi}_{X_i} \quad \text{and} \quad q^{\phi}_{X_i} = \mathrm{Dir}\left(\alpha + \mathrm{NN}^{\phi}(X_i)\right), \quad i = 1, \ldots, n,
    \end{align*}
    $\mathrm{NN}^{\phi} : \R^d \to \R_+^K$ is a neural network with $\phi \in \Phi$, $\Phi$ is the neural network class, with parameters
    \begin{equation}
    \label{eq:phi_hat_icd}
        \widehat{\phi}_n = \argmin_{\phi \in \Phi} \left\{ \mathrm{KL}\left( q^{\phi}_{X_{1:n}} \,\|\, \pi_{Y_{1:n}} \right) \right\}.
    \end{equation}
\end{definition}

Parameterizing the variational family via a neural network is convenient because it enables non-trivial predictions at arbitrary new covariate values $X_{n+1}$, as formalized in the following proposition.

\begin{proposition}
\label{prop:amortized_vi_icd_predictive}
    Given Definitions~\ref{def:icd_model} and~\ref{def:amortized_vi_icd}, the posterior predictive distribution is
    \begin{align*}
        Y_{n+1} \mid Y_{1:n}, X_{1:n+1} \approx \mathrm{Cat}\left( \frac{\alpha + \mathrm{NN}^{\widehat{\phi}_n}_{X_{n+1}}}{\alpha_0 + S^{\widehat{\phi}_n}_{X_{n+1}}} \right), \quad \text{where} \quad S^{\widehat{\phi}_n}_{X_{n+1}} = \sum_{k=1}^{K} \mathrm{NN}^{\widehat{\phi}_n}_{X_{n+1}}(k).
    \end{align*}
\end{proposition}

In particular, carrying out the VI in Definition~\ref{def:amortized_vi_icd} reduces to the following minimization problem.

\begin{proposition}
\label{prop:KL_expansion_icd_model}
    The optimization problem from Equation~\eqref{eq:phi_hat_icd} is equivalent to minimizing
    \begin{equation}
        \sum_{i=1}^{n} - \E_{q^{\phi}_{X_i}} \left[ \log f_{p_i}(Y_i) \right] + \mathrm{KL}\!\left( q^{\phi}_{X_i} \,\|\, \pi \right).
    \end{equation}
\end{proposition}

The objective in Proposition~\ref{prop:KL_expansion_icd_model} matches the EDL loss in Equation~\eqref{eq:EDL_loss_function} but lacks a regularization hyperparameter; we resolve this by introducing a temperature parameter into a new model.

\begin{definition}
\label{def:tempered_icd_model}
    \textbf{Tempered ICD Model.} The model is defined as follows:
    \begin{align*}
        Y_i \mid p_i &\indprop f_{p_{i}}^{\nu} = \prod_{k=1}^{K} p_{i}(k)^{\nu \, \mathbb{I}(Y_{i}=k)}, \quad i = 1, \ldots, n, \\
        p_i &\iid \pi = \text{Dir}(\alpha), \quad i = 1, \ldots, n,
    \end{align*}
    for some $\alpha \in \R_{+}^{K}$ and $\nu > 0$.
\end{definition}

This modified model is motivated by using a tempered likelihood \citep{bissiri2016general}, which provides explicit control over the data's influence on the posterior via the parameter $\nu$, thereby introducing an additional degree of freedom in the trade-off between data fit and regularization.

Notice that for $\nu \neq 1$, the categorical likelihood $f_{p_i}^{\nu}$ is unnormalized. A valid likelihood would require a normalizing constant %$\sum_{k=1}^K p_{i}(k)^{\nu}$
dependent on both $p_i$ and $\nu$. Instead of renormalizing, we follow a standard strategy in generalized Bayesian methods by treating the unnormalized likelihood as a pseudo-likelihood for posterior updating \citep{ventura2016pseudo}.

Under the tempered ICD model, the posterior and posterior predictive distributions remain tractable.

\begin{proposition}
\label{prop:tempered_icd_model_posterior}
    Given Definition~\ref{def:tempered_icd_model}, the posterior and posterior predictive distributions are
    \begin{align*}
        \pi^{\nu}_{Y_{i}} = \mathrm{Dir}(\alpha + \nu \, e_{Y_{i}}), \; i = 1, \ldots, n, \quad p_{1:n}\!\mid\!Y_{1:n} \sim \pi^{\nu}_{Y_{1:n}} = \prod_{i=1}^{n} \pi^{\nu}_{Y_{i}}, \quad Y_{n+1}\!\mid\!Y_{1:n} \sim \mathrm{Cat}\left(\frac{\alpha}{\alpha_{0}}\right).
    \end{align*}
\end{proposition}

Once again, we approximate this posterior via amortized VI to enable predictions at new inputs.

\begin{definition}
\label{def:amortized_vi_tempered_icd}
    \textbf{Amortized VI for the Tempered ICD Model.} Given Definition~\ref{def:tempered_icd_model}, we approximate the posterior via a mean-field variational family:
    \begin{align*}
        p_{1:n} \mid Y_{1:n} \approx q^{\widehat{\phi}_n}_{X_{1:n}} = \argmin_{q^{\phi}_{X_{1:n}}} \left\{ \mathrm{KL}\left( q^{\phi}_{X_{1:n}} \,\|\, \pi^{\nu}_{Y_{1:n}} \right) \right\},
    \end{align*}
    where $q^{\phi}_{X_{1:n}}, \mathrm{NN}^{\phi}$, and $\Phi$ are defined as in Definition~\ref{def:amortized_vi_icd}, and
    \begin{equation*}
        \widehat{\phi}_n = \argmin_{\phi \in \Phi} \left\{ \mathrm{KL}\left( q^{\phi}_{X_{1:n}} \,\|\, \pi^{\nu}_{Y_{1:n}} \right) \right\}
    \end{equation*}
    are the learned parameters.
\end{definition}

We are now ready to state the main theoretical result of this paper. Applying amortized VI to the Tempered ICD Model is exactly equivalent to EDL training under the loss in Equation~\eqref{eq:EDL_loss_function}.

\begin{theorem}
\label{thm:equivalence_edl_amortized_vi}
    Given Definitions~\ref{def:tempered_icd_model} and~\ref{def:amortized_vi_tempered_icd}, the amortized VI objective is equivalent to EDL training:
    \begin{equation}
    \label{eq:equivalence_edl_amortized_vi}
        \widehat{\phi}_n = \argmin_{\phi \in \Phi} \left\{ \mathcal{L}^{\frac{1}{\nu}}_{\mathrm{EDL}}\!\left(\phi; Y_{1:n}, X_{1:n}\right) \right\}. 
    \end{equation}
\end{theorem}

This statistical interpretation clarifies EDL's uncertainty mechanism: because the objective is a KL divergence, we can characterize the conditions satisfied by the learned parameters at convergence.

\begin{proposition}
\label{prop:perfect_interpolation_condition}
    Given Definitions~\ref{def:tempered_icd_model} and~\ref{def:amortized_vi_tempered_icd}, if $\Phi$ is sufficiently expressive and the optimization procedure attains the global minimum of the objective, then $\widehat{\phi}_n$ satisfies
    \begin{equation}
    \label{eq:perfect_interpolation_condition}
        \mathrm{NN}^{\widehat{\phi}_n}_{X_i} = \nu\, e_{Y_i}, \qquad i = 1,\ldots,n.  
    \end{equation} 
\end{proposition}

This result has a direct implication for uncertainty quantification.  
In particular, consider the following uncertainty measure introduced by \citet{sensoy2018evidential}.

\begin{definition}
\label{def:vacuity}
    \textbf{Vacuity.} Let $\alpha_{X} \in \mathbb{R}_+^K$ be the concentration parameters of a Dirichlet distribution possibly depending on covariates $X \in \mathbb{R}^d$. Then, its vacuity is defined as 
    \begin{equation}
        u(\alpha_{X}) = \frac{K}{a_0^X} \quad {\rm where} \quad a_0^X = \sum\limits_{k=1}^{K} \alpha_X(k).
    \end{equation} 
\end{definition}

In this setting, perfect interpolation implies that the uncertainty associated with in-sample probability vectors is constant across all training points and is fully determined by the temperature parameter $\nu$.

\begin{proposition}
\label{prop:vacuity_perfect_interpolation}
    Given Definitions~\ref{def:tempered_icd_model} and~\ref{def:amortized_vi_tempered_icd}, if the condition in Equation~\eqref{eq:perfect_interpolation_condition} is satisfied, then 
    \begin{equation}
        u \left(\alpha + \mathrm{NN}^{\widehat{\phi}_n}_{X_i}\right) \;=\; \frac{K}{\alpha_0 + \nu}, \qquad i = 1,\ldots,n. 
    \end{equation}
\end{proposition}
This demonstrates that EDL depends fundamentally on the temperature parameter $\nu$, rendering its uncertainty quantification spurious without meaningful tuning. These results corroborate recent theoretical critiques of EDL. Specifically, \citet{bengs2022pitfalls} and \citet{shen2024uncertainty} demonstrate that the reliance on a hyper-parameter to balance likelihood evidence against prior regularization introduces a fundamental arbitrariness into the uncertainty quantification process which can be theoretically problematic. For example, aleatoric uncertainty should represent a fixed constant of irreducible noise inherent to the data, yet in EDL, it changes depending on a user-defined hyperparameter. 

\subsection{ERM}
\label{sec:empirical_risk_minimization}

An alternative interpretation of the EDL framework arises from the ERM perspective.

\begin{definition}
\label{def:EDL_ERM}
    \textbf{EDL as ERM.}
    Given the DGP given in Equation~\eqref{eq:DGP}, consider the loss function
    \begin{align*}
        l^{\nu} (\phi; y, x) = \text{KL}\left( q^{\phi}_{x} \; || \; \pi^{\nu}_{y} \right),
    \end{align*}
    where $q^{\phi}_{x}$ is the variational distribution from Definition~\ref{def:amortized_vi_tempered_icd} and $\pi^{\nu}_{y}$ is the posterior from Definition~\ref{def:tempered_icd_model}. The parameter $\phi$ is learned by minimizing the empirical risk, with corresponding expected risk:
    \begin{align*}
        \widehat{\mathcal{R}}_{n}^{\nu} (\phi) = \frac{1}{n} \sum_{i=1}^{n} l^{\nu} (\phi; Y_{i}, X_{i}), \quad \mathcal{R}^{\nu} (\phi) = \E_{P^{*}_{Y,X}} \left[ l^{\nu} (\phi; Y, X) \right].
    \end{align*}
\end{definition}

This formulation clarifies the EDL objective as an empirical estimator of an expected value.

\begin{theorem}
\label{thm:equivalence_edl_erm}
    Given Definition~\ref{def:EDL_ERM}, the empirical risk and the EDL loss are proportional:
    \begin{align*}
        \widehat{\mathcal{R}}_{n}^{\nu}(\phi) \;=\; \frac{\nu}{n}\,\mathcal{L}^{\frac{1}{\nu}}_{\mathrm{EDL}}\!\left(\phi; Y_{1:n}, X_{1:n}\right).
    \end{align*}
\end{theorem}

This theorem clarifies the specific quantities targeted by the EDL training objective.

\begin{theorem}
\label{thm:oracle_variational_distribution}
    Given Definition~\ref{def:EDL_ERM}, if $\Phi$ is sufficiently expressive, it follows that
    \begin{align*}
        q^{\phi^{*}}_{X} = \mathrm{Dir}\left( \alpha + \nu \, P^{*}_{Y \mid X} \right) \quad \text{where} \quad
        \phi^{*} = \argmin_{\phi \in \Phi} \left\{ \mathcal{R}^{\nu} (\phi) \right\}.
    \end{align*}
\end{theorem}

Thus, even with the oracle $P^{*}_{Y \mid X}$, the uncertainty quantification remains governed by $\nu$.

\begin{proposition}
    Given Definition~\ref{def:EDL_ERM}, if the condition in Theorem~\ref{thm:oracle_variational_distribution} holds, then
    \begin{align*}
        u\left(\alpha + \nu \, P^{*}_{Y \mid X}\right) = \frac{K}{\alpha_0 + \nu}.
    \end{align*}
\end{proposition}
Thus, the temperature parameter dictates uncertainty regardless of data fit, and vacuity persists even as sample size grows. Our results align with \citet{shen2024uncertainty}, who also note that EDL does not distinguish epistemic from aleatoric uncertainty because vacuity remains constant regardless of the amount of data. Appendix~\ref{appdx:toy_example} provides an empirical validation of this result in a controlled toy setting. 

\section{Choice of the Temperature Parameter}
\label{sec:choice_temperature_parameter} 

The temperature parameter $\nu$ determines the uncertainty produced by EDL. Theorem~\ref{thm:oracle_variational_distribution} shows that EDL does not jointly learn the conditional label distribution and the associated uncertainty. It learns only the conditional distribution, while uncertainty is imposed externally through $\nu$. Consequently, although EDL is framed as optimizing both prediction and uncertainty quantification, it only resolves the predictive task, leaving uncertainty calibration fundamentally dependent on the choice of temperature. Given this insight, we avoid manual specification of $\nu$ or hyperparameter tuning on arbitrary predictive metrics. Instead, we seek a principled choice based on two desiderata:

\vspace{-2mm}
\begin{enumerate}
    \item \textbf{Distributional awareness:} Uncertainty is higher in low-density regions of the covariate space, reflecting limited information about the conditional label distribution. \vspace{-0mm}
    
    \item \textbf{Asymptotic consistency:} Epistemic uncertainty vanishes as the number of samples grows, whereas aleatoric uncertainty may persist due to inherent noise in the data. \vspace{-2mm}
\end{enumerate}

The intuition for these criteria stems from Theorem~\ref{thm:oracle_variational_distribution}, which suggests interpreting $\nu \, P^{*}_{Y \mid X}$ as a label count vector where $\nu$ acts as a pseudo-count. However, as a global parameter, $\nu$ fails to account for the sample size or the covariate distribution. To address this, we introduce a notion of covariate proximity so that similar observations can share information. This ensures that labels from nearby training points have lower uncertainty, while regions lacking nearby data reflect increased uncertainty.

Building on this local perspective, we consider a small neighborhood around a covariate value $x$ within a sample of size $n$. In a sufficiently small region, the expected number of observations is approximately $n \, P^{*}_{X}(x)$, and the expected number of labels belonging to class $k$ is $n \, P^{*}_{X}(x) \, P^{*}_{Y \mid X}(k \mid x)$. This relationship provides a principled foundation for choosing $\nu$.

\begin{definition}
\label{def:DIP_EDL}
    \textbf{DIP-EDL.}
    Given Definition~\ref{def:amortized_vi_tempered_icd}, we replace the variational family with \vspace{-2mm}
    \begin{equation}
        q^{\psi, \phi}_{X_i} = \mathrm{Dir}\left(\alpha + n \, \mathrm{DE}^{\psi}_{X_{i}} \, \mathrm{NN}^{\phi}_{X_{i}}\right), \quad i = 1, \ldots, n, 
    \end{equation}

    \vspace{-3mm}
    where $\mathrm{DE}^{\psi} : \R^d \to \R_+$ is a density estimator for $P^{*}_{X}$ with %parameters
    $\psi \in \Psi$, and $\mathrm{NN}^{\phi} : \R^d \to \R_+^K$ is a neural network estimator for $P^{*}_{Y \mid X}$ with %parameters
    $\phi \in \Phi$.
\end{definition}

\vspace{-2mm}
From an applied perspective, this specification is convenient for several reasons. First, it enables independent training of the density estimator and the neural network, making practical implementation straightforward. Second, it is architecture-agnostic: any density estimator and classifier can be used, and they may be trained with any loss functions suitable for density estimation and classification, provided they correctly learn the target distributions at least asymptotically. Crucially, while our exposition focuses on neural networks, DIP-EDL does not require the classifier to be differentiable. Unlike other EDL methods, which train a neural network end-to-end through the EDL loss and thus require gradient-based optimization, DIP-EDL accommodates any off-the-shelf classifier, including non-differentiable models (e.g., random forests), enabling principled uncertainty quantification for arbitrary classifier classes. Third, like other EDL-based approaches, DIP-EDL requires only a single forward pass at inference time, avoiding the computational overhead of methods that rely on multiple forward passes (e.g., deep ensembles and MC Dropout).

Beyond these practical properties, DIP-EDL also satisfies both proposed theoretical criteria. Regarding distributional awareness, the posterior predictive distribution contracts toward the prior in low-density regions where nearby training samples are scarce. Furthermore, the model achieves asymptotic consistency because $q_{X_{i}}$ concentrates around the true conditional distribution $P^{*}_{Y \mid X} ( \cdot \mid X_{i})$.

\begin{theorem}
\label{thm:asymptotic_consistency_DIP_EDL}
    Given the DGP~\eqref{eq:DGP} and Definition~\ref{def:DIP_EDL}, consistent estimation of $P^{*}_{X}$ and $P^{*}_{Y \mid X}$ implies pointwise asymptotic concentration in probability of the approximate posterior:\vspace{-2mm}
    \begin{align*}
        \mathrm{DE}^{\hat{\psi}_n}_{X_i} \overset{p}{\to} P^{*}_{X}(X_i), \quad \mathrm{NN}^{\hat{\phi}_n}_{X_i} \overset{p}{\to} P^{*}_{Y \mid X}(\cdot \mid X_i) \quad \Rightarrow \quad p_i \mid X_{1:n}, Y_{1:n} \approx q^{\hat{\psi}_n, \hat{\phi}_n}_{X_i} \overset{p}{\to} P^{*}_{Y \mid X}(\cdot \mid X_i).
    \end{align*}
\end{theorem}

\vspace{-2mm}
Prior work has shown that incorporating the inputs or latent representation density improves uncertainty estimates \citep[e.g.,][]{mukhoti2021deep,bui2024density,DAEDL,van2025hybridflow}. But, unlike these approaches, which treat covariate density as an ad-hoc component to enforce epistemic uncertainty, our method's dependence on it emerges naturally from the model formulation. This yields a more principled representation and superior empirical performance.

DIP-EDL conceptually resembles Posterior Network \citep[PostNet,][]{postnet}, as both rely on density-based pseudo-counts to specify the approximate posterior over class probabilities. However, the two methods differ in key aspects. First, PostNet uses a different factorization of the pseudo-counts. DIP-EDL decomposes the pseudo-count of observations with label $k$ in a neighborhood of $x$ as $n \, P^{*}_{X}(x) \, P^{*}_{Y \mid X}(k \mid x)$, while PostNet uses the opposite factorization $n \, P^{*}_{Y}(k) \, P^{*}_{X \mid Y}(x \mid k)$, where $P^{*}_{Y}$ is the marginal class distribution and $P^{*}_{X \mid Y}$ is the class-conditional density in the input space. More importantly, PostNet estimates $P^{*}_{X \mid Y}$ via $P^{*}_{Z \mid Y}$, where $Z$ is a learned latent representation of the input, which empirically leads to degraded performance compared to our decomposition. %When $Z(X)$ is a bijection, the two formulations are equivalent up to a change of variables.

With that being said, despite this connection, our approach offers three primary advantages: \vspace{-1mm}

\begin{itemize}
    \item \textbf{Direct Calibration:} It ties uncertainty directly to observable predictive outputs, resulting in improved empirical performance (see Section \ref{sec:experiment_real}).
    
    \item \textbf{Modularity:} Unlike PostNet, which relies on a single training objective to learn the covariate density and conditional label distribution, DIP-EDL estimates these components separately. This separation allows for a more flexible and modular training process.
    
    \item \textbf{Efficiency:} DIP-EDL has the potential to be much more computationally efficient as it requires only a single density estimator for the entire input space rather than one per class, which can be practically challenging when the number of classes is large. 
\end{itemize}

\vspace{-2mm}
\section{Experiments}
\label{sec:experiments}

% We evaluate DIP-EDL against established baselines for OOD detection under distributional shift.

\subsection{Experimental Configuration}
\vspace{-2mm}

\paragraph{Datasets.} We use \textbf{MNIST} \cite{mnist} and \textbf{CIFAR-10} \cite{cifar10} as ID datasets. For OOD evaluation, we use \textbf{K-MNIST} \cite{kmnist} and \textbf{Omniglot} \cite{Omniglot} for MNIST, and \textbf{CIFAR-100} \cite{CIFAR100} and \textbf{SVHN} \cite{svhn} for CIFAR-10, covering both near-OOD (digits vs.\ characters) and far-OOD (digits vs.\ natural images) scenarios. % Please see Appendix~\ref{appdx:mnist_results} for experimental results when MNIST is ID dataset. 
Additionally, we perform experiments on a real scientific dataset \textbf{LAMOST Data Release 9 (DR9)} \cite{cui2012large}. We treat Galaxy and Quasar classes as ID, while we leave the class star as OOD. 

\vspace{-2mm}
\paragraph{Baselines.} We benchmark against four Dirichlet-based baselines: \textbf{EDL} \cite{sensoy2018evidential}, \textbf{R-EDL} \cite{REDL}, \textbf{Re-EDL} \cite{chen2025revisiting}, \textbf{DAEDL} \cite{DAEDL}, and \textbf{PostNet} \cite{postnet}, which require no OOD training data and operate in a single forward pass, similarly to our method. We additionally compare DIP-EDL against \textbf{MC-Dropout} \cite{gal2016dropout} and \textbf{Deep Ensembles} \cite{lakshminarayanan2017simple}, which require multiple forward passes at inference time. 

% These methods, like ours, do not leverage OOD data during training, ensuring a fair comparison. Among approaches that satisfy this constraint, they report competitive or state-of-the-art performance in prior work.
\vspace{-2mm}
\paragraph{Implementation Details.} We tailor model architectures to each dataset. For MNIST, we use \textbf{LeNet-5} \cite{LeNet} following \cite{sensoy2018evidential}, with a Masked Autoregressive Flow (\textbf{MAF}) \cite{MAF} on the flattened pixel space for density estimation. 
For CIFAR-10, we use \textbf{WideResNet-28-10} \cite{zagoruyko2016wide} with Gaussian Discriminant Analysis (\textbf{GDA}) \cite{GDA} fitted on the resulting 640-dimensional embeddings.
For LAMOST, we use a \textbf{multi-branch 1D CNN} (\cite{lamost_spectra_classifier}) designed for stellar spectra classification, with \textbf{GDA} fitted on 32-dimensional penultimate features. 
See Appendix~\ref{sec:appendix_experiments} for details.

\vspace{-2mm}
\paragraph{Performance Metrics.} Following the OOD detection literature \cite{REDL, DAEDL, postnet}, we report three metrics: \textbf{classification accuracy} on the ID test set; \textbf{AUROC} (Area Under the Receiver Operating Characteristic curve) and \textbf{AUPR} (Area Under the Precision-Recall curve), which rank OOD samples (positive class) higher in uncertainty than ID samples (negative class). The Brier Score (\textbf{BS}) additionally measures predicted probability magnitudes against one-hot targets for ID data and a uniform distribution for OOD data, penalizing high-confidence predictions on unseen domains even when correctly ranked as uncertain.
\vspace{-2mm}

\subsection{Experiments with MNIST as ID dataset} \label{sec:experiment_real}

On MNIST, Table~\ref{tab:mnist_results}, \textbf{DIP-EDL} achieves the highest classification accuracy ($99.53$\%) and lowest ID Brier Score ($0.01$) across all methods, including MC-Dropout and Deep Ensembles, improving calibration over EDL ($0.19$) and DAEDL ($0.02$).

% Table~\ref{tab:mnist_results} summarizes performance on the MNIST ID task, together with robustness under semantic (K-MNIST) and domain (Omniglot) shifts. \textbf{DIP-EDL} achieves the highest classification accuracy ($99.53$\%) and lowest ID Brier Score ($0.01$) across all methods, including MC-Dropout and Deep Ensembles, improving calibration over EDL ($0.19$) and DAEDL ($0.02$).

These improvements extend to OOD detection. DIP-EDL achieves AUROC of $0.99$ on both OOD datasets, outperforming all baselines. On Omniglot, AUROC is near-saturated across methods, but DIP-EDL yields the lowest OOD BS, unlike PostNet, MC-Dropout, and Deep Ensembles, which exhibit high OOD BS despite strong AUROC, indicating overconfidence under distributional shift.

\begin{table*}[ht]
    \centering
    \caption{\textbf{MNIST: ID and OOD performance.} ID accuracy and calibration, with OOD detection against K-MNIST (near-OOD) and Omniglot (far-OOD). Results are $\mu \pm \sigma$ over 4 runs.} \vspace{-2mm}
    % FIX 1: Reduce column padding
    \setlength{\tabcolsep}{3pt} 
    % FIX 2: Resize entire table
    \resizebox{\textwidth}{!}{
        \begin{tabular}{@{} l | c c | c c c c c c @{}}
            \toprule
            \textbf{Model} & \multicolumn{2}{c|}{\textbf{ID Performance}} & \multicolumn{6}{c}{\textbf{OOD Performance Metrics}} \\ 
            \cmidrule(lr){2-3} \cmidrule(lr){4-9}

            % --- METRIC TYPE ROW ---
            & \multicolumn{1}{c}{Acc. ($\uparrow$)} & \multicolumn{1}{c|}{BS ($\downarrow$)} 
            & \multicolumn{2}{c}{AUROC ($\uparrow$)} 
            & \multicolumn{2}{c}{AUPR ($\uparrow$)} 
            & \multicolumn{2}{c}{OOD BS ($\downarrow$)} \\ 
            \cmidrule(lr){4-5} \cmidrule(lr){6-7} \cmidrule(lr){8-9}

            % --- DATASET ROW ---
            & \multicolumn{2}{c|}{\textbf{MNIST}} & \multicolumn{1}{c}{\textbf{K-MNIST}} & \multicolumn{1}{c}{\textbf{Omniglot}} & \multicolumn{1}{c}{\textbf{K-MNIST}} & \multicolumn{1}{c}{\textbf{Omniglot}} & \multicolumn{1}{c}{\textbf{K-MNIST}} & \multicolumn{1}{c}{\textbf{Omniglot}} \\
            \midrule
            
            \textbf{EDL} & $0.9166 \pm 0.0498$ & $0.1886 \pm 0.0434$ & $0.9116 \pm 0.0340$ & $0.9204 \pm 0.0477$ & $0.8622 \pm 0.0588$ & $0.8566 \pm 0.0778$ & $0.0569 \pm 0.0080$ & $0.0185 \pm 0.0032$ \\
            \textbf{R-EDL} & $0.9940 \pm 0.0005$ & $0.0113 \pm 0.0009$ & $0.9729 \pm 0.0014$ & $0.9879 \pm 0.0026$ & $0.9599 \pm 0.0025$ & $0.9818 \pm 0.0055$ & $0.1755 \pm 0.0192$ & $0.0726 \pm 0.0372$ \\
            \textbf{Re-EDL} & $0.9940 \pm 0.0010$ & $0.0100 \pm 0.0010$ & $0.9910 \pm 0.0010$ & $\mathbf{1.0000 \pm 0.0000}$ & $0.9890 \pm 0.0020$ & $\mathbf{1.0000 \pm 0.0000}$ & $0.7050 \pm 0.0160$ & $0.5430 \pm 0.0310$ \\
            \textbf{DAEDL} & $0.9937 \pm 0.0005$ & $0.0225 \pm 0.0014$ & $0.9987 \pm 0.0003$ & $0.9997 \pm 0.0002$ & $0.9981 \pm 0.0006$ & $0.9995 \pm 0.0002$ & $0.0097 \pm 0.0018$ & $\mathbf{0.0000 \pm 0.0000}$ \\
            \textbf{PostNet} & $0.9925 \pm 0.0009$ & $0.0117 \pm 0.0017$ & $0.9602 \pm 0.0128$ & $0.9936 \pm 0.0026$ & $0.9570 \pm 0.0111$ & $0.9862 \pm 0.0053$ & $0.4922 \pm 0.0184$ & $0.2816 \pm 0.0514$ \\
            \textbf{MC Dropout} & $0.9940 \pm 0.0000$ & $0.0090 \pm 0.0010$ & $0.9800 \pm 0.0020$ & $0.9880 \pm 0.0010$ & $0.9760 \pm 0.0020$ & $0.9820 \pm 0.0030$ & $0.6150 \pm 0.0060$ & $0.4800 \pm 0.0390$ \\
            \textbf{Deep Ensemble} & $0.9940 \pm 0.0000$ & $0.0100 \pm 0.0010$ & $0.9830 \pm 0.0010$ & $0.9970 \pm 0.0010$ & $0.9810 \pm 0.0010$ & $0.9980 \pm 0.0010$ & $0.6090 \pm 0.0060$ & $0.3840 \pm 0.0460$ \\
            \textbf{DIP-EDL} & $\mathbf{0.9953 \pm 0.0003}$ & $\mathbf{0.0081 \pm 0.0005}$ & $\mathbf{0.9997 \pm 0.0000}$ & $\mathbf{0.9998 \pm 0.0000}$ & $\mathbf{0.9995 \pm 0.0001}$ & $\mathbf{0.9996 \pm 0.0001}$ & $\mathbf{0.0013 \pm 0.0002}$ & $\mathbf{0.0000 \pm 0.0000}$ \\
            \bottomrule
            \addlinespace
        \end{tabular}
    }
    \vspace{-5mm}
    \label{tab:mnist_results}
\end{table*}

\subsection{Experiments with CIFAR-10 as ID dataset}
\label{sec:experiment_real_cifar}

On CIFAR-10 (Table~\ref{tab:cifar_results}) a more challenging setting compared to MNIST, DIP-EDL achieves the highest accuracy ($95.03\%$) and lowest ID BS ($0.0978$) among single-pass methods. Deep Ensembles, using multiple forward passes, achieves better ID performance at the cost of inference efficiency.

% Table~\ref{tab:cifar_results} reports ID performance on CIFAR-10 and OOD detection results on CIFAR-100 (near semantic shift) and SVHN (far domain shift). Compared to MNIST, CIFAR-10 is more challenging due to higher visual complexity and greater intra-class variability. Among single-pass methods, DIP-EDL achieves the highest accuracy ($95.03\%$) and lowest ID BS ($0.0978$). Deep Ensembles, using multiple forward passes, achieves better ID performance ($96.14\%$ accuracy, $0.0587$ BS) at the cost of inference efficiency.

For OOD detection, performance trends differ between near- and far-shift scenarios. On the far-OOD task (SVHN), DIP-EDL achieves the highest AUROC ($0.9660$) and AUPR ($0.9830$), outperforming all baselines including Deep Ensembles ($0.9628$ AUROC, $0.9802$ AUPR). On the near-OOD task (CIFAR-100), which shares significant visual overlap with CIFAR-10 and is thus more challenging, scores are more closely matched with Deep Ensembles, which achieves the highest AUROC ($0.9072$), while DIP-EDL achieves the second highest AUROC ($0.9002$) and the highest AUPR ($0.8859$). These results confirm reliable distributional shift detection by DIP-EDL, particularly for far-OOD scenarios.

% Despite strong AUROC and AUPR, DIP-EDL exhibits higher OOD BS on CIFAR-10 than on MNIST. Since OOD BS measures how closely OOD predictions approach a uniform distribution (the ideal case of total epistemic uncertainty), higher values indicate overconfidence rather than poor ranking. We attribute this to the difficulty of density estimation in high-dimensional spaces: on complex datasets, likelihood estimates are insufficiently small to fully offset the classifier's confidence, unlike in the MNIST setting (see Appendix~\ref{appdx:on_BS}). This is corroborated by the ablation study (Section~\ref{sec:ablation_study}), where removing the training set size $n$ improves OOD BS without affecting AUROC or AUPR, confirming that high OOD BS is a density estimation artifact, not poor OOD detection.

DIP-EDL also achieves the second-lowest OOD BS across both OOD datasets, surpassing Deep Ensembles. This reflects well-calibrated epistemic uncertainty, arising from accurate feature representation in lower-dimensional feature space, together with precise density estimation on it, which keeps likelihoods sufficiently small to offset classifier confidence. EDL achieves the lowest OOD BS but at the cost of poor AUROC and AUPR, indicating uninformative uniform predictions rather than meaningful uncertainty. DIP-EDL is the only method that simultaneously achieves strong ranking (AUROC/AUPR) \emph{and} strong calibration (OOD BS) on out-of-distribution data. 
The impact of individual components is corroborated by the ablation study (Section~\ref{sec:ablation_study}), and the effect of corrupted likelihood estimates on the overall DIP-EDL performance is discussed in Appendix~\ref{appdx:density_robustness}. Difficulty of density estimation in MNIST versus CIFAR-10 experiments is visualized in Appendix~\ref{appdx:on_BS}. \vspace{-2mm}

% DIP-EDL also achieves the second-lowest OOD BS on CIFAR-10 across both OOD datasets, outperforming Deep Ensembles which is competitive across other metrics. This reflects well-calibrated epistemic uncertainty on OOD data and is a result of a combination of accurate density estimation as well as accurate feature representation in lower space used for density estimation. This makes likelihood estimates sufficiently small to offset the classifier's confidence. The impact of individual components is further corroborated by the ablation study (Section~\ref{sec:ablation_study}), while the effect of poor and corrupted likelihood estimates on the overall DIP-EDL performance is discussed in Appendix~\ref{appdx:density_robustness}. Difficulty of density estimation in MNIST versus CIFAR-10 experiments is visualized in Appendix~\ref{appdx:on_BS}. Note that EDL achieves the lowest OOD BS but at the cost of poor AUROC and AUPR ($0.759$ and $0.690$ on CIFAR-100), indicating uninformative uniform predictions rather than meaningful uncertainty. DIP-EDL is the only method that simultaneously achieves strong ranking (AUROC/AUPR) \emph{and} strong calibration (OOD BS) on out-of-distribution data. \vspace{-2mm}

\begin{table*}[ht]
    \centering
    \caption{\textbf{CIFAR-10: ID and OOD performance.} ID accuracy and calibration, with OOD detection against CIFAR-100 (near-OOD) and SVHN (far-OOD). Results are $\mu \pm \sigma$ over 4 runs.} \vspace{-2mm}
    % FIX 1: Reduce column padding to squeeze content
    \setlength{\tabcolsep}{3pt} 
    % FIX 2: Resize entire table to fit text width
    \resizebox{\textwidth}{!}{
        \begin{tabular}{@{} l | c c | c c c c c c @{}}
            \toprule
            \textbf{Model} & \multicolumn{2}{c|}{\textbf{ID Performance}} & \multicolumn{6}{c}{\textbf{OOD Performance Metrics}} \\ 
            \cmidrule(lr){2-3} \cmidrule(lr){4-9}

            % --- METRIC TYPE ROW ---
            & \multicolumn{1}{c}{Acc. ($\uparrow$)} & \multicolumn{1}{c|}{BS ($\downarrow$)} 
            & \multicolumn{2}{c}{AUROC ($\uparrow$)} 
            & \multicolumn{2}{c}{AUPR ($\uparrow$)} 
            & \multicolumn{2}{c}{OOD BS ($\downarrow$)} \\ 
            \cmidrule(lr){4-5} \cmidrule(lr){6-7} \cmidrule(lr){8-9}

            % --- DATASET ROW ---
            & \multicolumn{2}{c|}{\textbf{CIFAR-10}} & \multicolumn{1}{c}{\textbf{CIFAR-100}} & \multicolumn{1}{c}{\textbf{SVHN}} & \multicolumn{1}{c}{\textbf{CIFAR-100}} & \multicolumn{1}{c}{\textbf{SVHN}} & \multicolumn{1}{c}{\textbf{CIFAR-100}} & \multicolumn{1}{c}{\textbf{SVHN}} \\
            \midrule
            
            \textbf{EDL} & $0.7535 \pm 0.1526$ & $0.3520 \pm 0.1355$ & $0.7588 \pm 0.0880$ & $0.7519 \pm 0.1465$ & $0.6900 \pm 0.0956$ & $0.8094 \pm 0.0992$ & $\mathbf{0.0960 \pm 0.0326}$ & $\mathbf{0.0492 \pm 0.0067}$ \\  % 4 seeds [10, 20, 30, 40]
            \textbf{R-EDL} & $0.8957 \pm 0.0030$ & $0.1737 \pm 0.0040$ & $0.8484 \pm 0.0040$ & $0.8741 \pm 0.0101$ & $0.8016 \pm 0.0051$ & $0.9150 \pm 0.0096$ & $0.3603 \pm 0.0123$ & $0.3262 \pm 0.0396$ \\  % 4 seeds [10, 20, 30, 40]
            \textbf{Re-EDL} & $0.8939 \pm 0.0009$ & $0.1737 \pm 0.0020$ & $0.8603 \pm 0.0020$ & $0.9102 \pm 0.0059$ & $0.8229 \pm 0.0039$ & $0.9427 \pm 0.0058$ & $0.6613 \pm 0.0108$ & $0.6303 \pm 0.0173$ \\  % 4 seeds [10, 20, 30, 40]
            \textbf{DAEDL} & $0.8879 \pm 0.0024$ & $0.1760 \pm 0.0026$ & $0.8323 \pm 0.0015$ & $0.8606 \pm 0.0135$ & $0.7913 \pm 0.0015$ & $0.9088 \pm 0.0153$ & $0.4087 \pm 0.0079$ & $0.3843 \pm 0.0426$ \\  % 4 seeds [10, 20, 30, 40]
            \textbf{PostNet} & $0.8169 \pm 0.0127$ & $0.2623 \pm 0.0159$ & $0.7691 \pm 0.0113$ & $0.7760 \pm 0.0342$ & $0.7254 \pm 0.0139$ & $0.8532 \pm 0.0229$ & $0.4609 \pm 0.0271$ & $0.4734 \pm 0.0687$ \\  % 4 seeds [10, 20, 30, 40]
            \textbf{MC Dropout} & $\mathit{0.9534 \pm 0.0011}$ & $\mathit{0.0729 \pm 0.0008}$ & $0.8833 \pm 0.0032$ & $0.8802 \pm 0.0223$ & $0.8407 \pm 0.0034$ & $0.9053 \pm 0.0153$ & $0.5958 \pm 0.0059$ & $0.6111 \pm 0.0513$ \\  % 4 seeds [10, 20, 30, 40]
            \textbf{Deep Ensemble} & $\mathbf{0.9614 \pm 0.0006}$ & $\mathbf{0.0587 \pm 0.0007}$ & $\mathbf{0.9072 \pm 0.0009}$ & $\mathit{0.9628 \pm 0.0056}$ & $\mathit{0.8807 \pm 0.0009}$ & $\mathit{0.9802 \pm 0.0037}$ & $0.5222 \pm 0.0016$ & $0.4200 \pm 0.0302$ \\  % 4 seeds [10, 20, 30, 40]
            \textbf{DIP-EDL} & $0.9503 \pm 0.0017$ & $0.0978 \pm 0.0014$ & $\mathit{0.9002 \pm 0.0006}$ & $\mathbf{0.9660 \pm 0.0034}$ & $\mathbf{0.8859 \pm 0.0003}$ & $\mathbf{0.9830 \pm 0.0028}$ & $\mathit{0.3251 \pm 0.0044}$ & $\mathit{0.1109 \pm 0.0088}$ \\  % 4 seeds [10, 20, 30, 40]
            \bottomrule
            \addlinespace
        \end{tabular}
    }
    \vspace{-4mm}
    \label{tab:cifar_results}
\end{table*}

\subsection{Experiments with LAMOST as ID dataset}
\label{sec:experiment_lamost}

On LAMOST stellar spectra dataset (Table~\ref{tab:lamost_results}; dimension 3000), the ID task is Galaxy/Quasar classification with Star spectra as OOD. LAMOST is the most challenging setting in this benchmark, with severe class imbalance causing several methods to degenerate. EDL exhibits near-random accuracy with high variance across seeds, while R-EDL and DAEDL plateau at the Quasar fraction in the training set, indicating majority-class collapse. Among the remaining methods, DIP-EDL ties Deep Ensembles for the highest accuracy ($0.8929$) and achieves the lowest ID Brier Score ($0.1555$) among single-model approaches, surpassed marginally by Deep Ensembles.

% Table~\ref{tab:lamost_results} reports results on the LAMOST stellar spectra dataset (with dimension 3000), where the ID task is binary classification between Galaxy and Quasar spectra, and Star spectra serve as OOD. LAMOST represents the most challenging setting in this benchmark, characterized by severe class imbalance which causes several methods to degenerate. EDL exhibits near-random accuracy with high variance across seeds, while R-EDL and DAEDL plateau at the Quasar fraction in the training set, indicating majority-class collapse. Among the remaining methods, DIP-EDL ties Deep Ensembles with the highest accuracy ($0.8929$) and achieves the lowest ID Brier Score ($0.1555$) among single-model approaches, surpassed only marginally by Deep Ensembles.

OOD scores are moderate across all methods, reflecting two inherent difficulties of this task: Star spectra share low-level spectral features with the ID classes, and the limited effective training set size constrains the quality of density estimates. Deep Ensembles achieves the highest AUROC, with DIP-EDL second ($0.6789$), both well above the remaining baselines. AUPR is uniformly high across methods due to the large proportion of Star spectra, making it an uninformative discriminator in this setting. PostNet achieves the lowest OOD BS but near-random AUROC, paralleling EDL on CIFAR-10. DIP-EDL is the strongest single-model method on this benchmark, achieving competitive ID classification and OOD detection performance on a real-world dataset where most baselines fail.

% OOD detection scores are moderate across all methods, reflecting the inherent difficulty of this task due to Star spectra sharing low-level spectral features with the ID classes, as well as the limited effective training set size constraining the quality of density estimates. Deep Ensembles achieves the highest AUROC ($0.7425$), with DIP-EDL second ($0.6789$), both well above the remaining baselines ($\text{AUROC} \leq 0.6022$). AUPR is uniformly high across methods due to the large proportion of Star spectra in LAMOST, making it an uninformative discriminator in this setting. PostNet achieves the lowest OOD BS ($0.0297$) but near-random AUROC ($0.532$), paralleling the EDL behavior on CIFAR-10. DIP-EDL is the strongest single-model method on this benchmark, achieving competitive ID classification and OOD detection performance on a real-world dataset where most baselines fail.
    \vspace{-1mm}

% ── LAMOST ──────────────────────────────────────────
\begin{table*}[ht]
    \centering
    \caption{\textbf{LAMOST: ID and OOD performance.} ID accuracy and calibration, with OOD detection against Star spectra. Results are $\mu \pm \sigma$ over 4 runs.} \vspace{-2mm}
    \setlength{\tabcolsep}{3pt}
    \resizebox{0.98\textwidth}{!}{
        \begin{tabular}{@{} l | c c | c c c @{}}
            \toprule
            \textbf{Model} & \multicolumn{2}{c|}{\textbf{ID Performance}} & \multicolumn{3}{c}{\textbf{OOD Performance Metrics}} \\
            \cmidrule(lr){2-3} \cmidrule(lr){4-6}
            & \multicolumn{1}{c}{Acc. ($\uparrow$)} & \multicolumn{1}{c|}{BS ($\downarrow$)}
            & \multicolumn{1}{c}{AUROC ($\uparrow$)} & \multicolumn{1}{c}{AUPR ($\uparrow$)} & \multicolumn{1}{c}{OOD BS ($\downarrow$)} \\
            \cmidrule(lr){4-4} \cmidrule(lr){5-5} \cmidrule(lr){6-6}
            & \multicolumn{2}{c|}{\textbf{LAMOST}} & \multicolumn{1}{c}{\textbf{Star}} & \multicolumn{1}{c}{\textbf{Star}} & \multicolumn{1}{c}{\textbf{Star}} \\
            \midrule
            \textbf{EDL} & $0.4955 \pm 0.2707$ & $0.7585 \pm 0.4638$ & $0.4934 \pm 0.1710$ & $0.9884 \pm 0.0053$ & $0.3595 \pm 0.2033$ \\  % 4 seeds [10, 20, 30, 40]
            \textbf{R-EDL} & $0.7706 \pm 0.0096$ & $0.4273 \pm 0.0728$ & $0.6022 \pm 0.1055$ & $0.9933 \pm 0.0026$ & $0.0783 \pm 0.0796$ \\  % 4 seeds [10, 20, 30, 40]
            \textbf{Re-EDL} & $0.7706 \pm 0.0096$ & $0.4272 \pm 0.0728$ & $0.5957 \pm 0.0990$ & $0.9932 \pm 0.0024$ & $\mathit{0.0752 \pm 0.0764}$ \\  % 4 seeds [10, 20, 30, 40]
            \textbf{DAEDL} & $0.7706 \pm 0.0096$ & $0.2802 \pm 0.0028$ & $0.5748 \pm 0.0138$ & $0.9904 \pm 0.0004$ & $0.1418 \pm 0.0054$ \\  % 4 seeds [10, 20, 30, 40]
            \textbf{PostNet} & $0.7328 \pm 0.0449$ & $0.4231 \pm 0.0662$ & $0.5315 \pm 0.1036$ & $0.9913 \pm 0.0027$ & $\mathbf{0.0297 \pm 0.0226}$ \\  % 4 seeds [10, 20, 30, 40]
            \textbf{MC Dropout} & $0.5851 \pm 0.2215$ & $0.6055 \pm 0.2897$ & $0.4974 \pm 0.1334$ & $0.9886 \pm 0.0051$ & $0.3729 \pm 0.1132$ \\  % 4 seeds [10, 20, 30, 40]
            \textbf{Deep Ensemble} & $\mathbf{0.8929 \pm 0.0133}$ & $\mathbf{0.1497 \pm 0.0076}$ & $\mathbf{0.7425 \pm 0.0191}$ & $\mathbf{0.9957 \pm 0.0006}$ & $0.2394 \pm 0.0377$ \\  % 4 seeds [10, 20, 30, 40]
            \textbf{DIP-EDL} & $\mathbf{0.8929 \pm 0.0115}$ & $\mathit{0.1555 \pm 0.0059}$ & $\mathit{0.6789 \pm 0.0389}$ & $\mathit{0.9948 \pm 0.0010}$ & $0.2517 \pm 0.0432$ \\  % 4 seeds [10, 20, 30, 40]
            \bottomrule
        \end{tabular}
    }
    \label{tab:lamost_results}
    \vspace{-4mm}
\end{table*}

\subsection{Ablation Study}
\label{sec:ablation_study}

We ablate DIP-EDL's three components to isolate their contributions: (i) training set size, (ii) the density estimator, and (iii) the discriminative classifier. The results (Table~\ref{tab:ablation_main}) confirm the distinct role of each component: excluding the discriminative classifier yields near-random accuracy, while excluding the density estimator yields near-random OOD detection. Full results are in Appendix~\ref{appdx:ablation_full}.

Specifically, the training set size acts as a scaling factor required for concentration (Theorem~\ref{thm:asymptotic_consistency_DIP_EDL}), with no effect on ID accuracy or OOD ranking. Scaling by $n$ increases Dirichlet concentration parameters proportionally, improving ID calibration (ID BS) while amplifying OOD deviation from the uniform prior (OOD BS), as seen in configurations 3--4 of Table~\ref{tab:ablation_main}. Without scaling by the sample size and the density estimator, ID calibration degrades even when ID accuracy is preserved (Appendix~\ref{appdx:ablation_full}). The effect of $n$ on OOD BS is more pronounced on CIFAR-10 than on MNIST, reflecting harder density estimation in higher-dimensional spaces and underscoring the importance of a high-quality density estimator, whose accurate likelihoods minimize the scaler's effect. Additional sensitivity analyses on the density scale parameter $\gamma$ and on corrupted density estimates are provided in Appendix~\ref{appdx:density_robustness}, confirming robustness to density estimation imprecision. \vspace{-2mm}

% Particularly, we note that the training set size acts as a scaling factor required for concentration (Theorem~\ref{thm:asymptotic_consistency_DIP_EDL}), with no effect on ID accuracy or OOD ranking. Scaling by $n$ increases Dirichlet concentration parameters proportionally, improving ID calibration (ID BS) while amplifying OOD deviation from the uniform prior (OOD BS), as seen in configurations 3--4 of Table~\ref{tab:ablation_main}. We observe that without scaling by the sample size and the density estimator, ID calibration degrades even when ID accuracy is preserved (Appendix~\ref{appdx:ablation_full}). The effect of $n$ on OOD BS is more pronounced on CIFAR-10 than on MNIST, reflecting harder density estimation in higher-dimensional spaces and underscoring the importance of a high-quality density estimator, whose accurate likelihoods minimize the scaler's effect. Additional sensitivity analyses on the density scale parameter $\gamma$ and on corrupted density estimates are provided in Appendix~\ref{appdx:density_robustness}, confirming robustness to density estimation imprecision.

\begin{table*}[th]
    \centering
    \caption{\textbf{Ablation study of DIP-EDL.} Individual and pairwise contributions of training set size, the density estimator, and the discriminative classifier on MNIST and CIFAR-10.} \vspace{-2mm}
    \setlength{\tabcolsep}{3pt} 
    \small 
    \begin{adjustbox}{max width=0.98\textwidth, center}
        \begin{tabular}{@{} ccc | cc | cc | cc | cc @{}}
            \toprule
            \multicolumn{3}{c|}{\textbf{Components}} & \multicolumn{2}{c|}{\textbf{ID Performance}} & \multicolumn{6}{c}{\textbf{OOD Performance Metrics}} \\ 
            \cmidrule(r){1-3} \cmidrule(lr){4-5} \cmidrule(l){6-11}

            % --- METRIC TYPE ROW ---
            & & & \multicolumn{1}{c}{Acc. ($\uparrow$)} & \multicolumn{1}{c|}{BS ($\downarrow$)} 
            & \multicolumn{2}{c}{AUROC ($\uparrow$)} 
            & \multicolumn{2}{c}{AUPR ($\uparrow$)} 
            & \multicolumn{2}{c}{OOD BS ($\downarrow$)} \\ 
            \cmidrule(lr){4-5} \cmidrule(lr){6-7} \cmidrule(lr){8-9} \cmidrule(lr){10-11}

            % --- DATASET ROW ---
            $n$ & $\mathrm{DE}^{\psi}_{X_{i}}$ & $\mathrm{NN}^{\phi}_{X_{i}}$ & \multicolumn{2}{c|}{\textbf{MNIST}} & \textbf{K-MNIST} & \textbf{Omniglot} & \textbf{K-MNIST} & \textbf{Omniglot} & \textbf{K-MNIST} & \textbf{Omniglot} \\
            \midrule
            
            % --- MNIST SECTION ---
            \checkmark & \checkmark & $\times$ & $0.0980$ & $0.9000$ & $\mathbf{0.9998}$ & $\mathbf{0.9998}$ & $\mathbf{0.9996}$ & $\mathbf{0.9997}$ & $\mathbf{0.0000}$ & $\mathbf{0.0000}$ \\
            \checkmark & $\times$ & \checkmark & $\mathbf{0.9958}$ & $\mathbf{0.0069}$ & $0.5138$ & $0.5299$ & $0.5742$ & $0.6408$ & $0.6826$ & $0.7124$ \\
            $\times$ & \checkmark & \checkmark & $0.9952$ & $0.7202$ & $0.9996$ & $0.9996$ & $0.9992$ & $0.9994$ & $\mathbf{0.0000}$ & $\mathbf{0.0000}$ \\
            \cmidrule(lr){1-11}
            \checkmark & \checkmark & \checkmark & $0.9955$ & $0.0079$ & $\mathbf{0.9998}$ & $\mathbf{0.9998}$ & $0.9995$ & $\mathbf{0.9997}$ & $0.0014$ & $\mathbf{0.0000}$ \\
            
            \midrule \midrule 

            % --- CIFAR-10 DATASET ROW ---
            & & & \multicolumn{2}{c|}{\textbf{CIFAR-10}} & \textbf{CIFAR-100} & \textbf{SVHN} & \textbf{CIFAR-100} & \textbf{SVHN} & \textbf{CIFAR-100} & \textbf{SVHN} \\
            \cmidrule(lr){4-11}

            % --- CIFAR SECTION ---
            \checkmark & \checkmark & $\times$ & $0.1000$ & $0.9000$ & $\mathbf{0.9000}$ & $\mathbf{0.9686}$ & $\mathbf{0.8864}$ & $\mathbf{0.9850}$ & $\mathbf{0.0000}$ & $\mathbf{0.0000}$ \\
            \checkmark & $\times$ & \checkmark & $\mathbf{0.9496}$ & $\mathbf{0.0831}$ & $0.5036$ & $0.5014$ & $0.5014$ & $0.7224$ & $0.6945$ & $0.6412$ \\
            $\times$ & \checkmark & \checkmark & $\mathbf{0.9496}$ & $0.7653$ & $\mathbf{0.9000}$ & $\mathbf{0.9686}$ & $\mathbf{0.8864}$ & $\mathbf{0.9850}$ & $0.0004$ & $\mathbf{0.0000}$ \\
            \cmidrule(lr){1-11}
            \checkmark & \checkmark & \checkmark & $\mathbf{0.9496}$ & $0.0978$ & $\mathbf{0.9000}$ & $\mathbf{0.9686}$ & $\mathbf{0.8864}$ & $\mathbf{0.9850}$ & $0.3303$ & $0.1032$ \\
            \bottomrule
        \end{tabular}
    \end{adjustbox}
    \label{tab:ablation_main} 
    \vspace{-2mm}
\end{table*}

    \vspace{-1mm}
\subsection{ID Calibration}
    \vspace{-1mm}

Standard neural networks often suffer from miscalibration, even with regularization \citep{guo2017calibration, vashistha2025calibration}. Conversely, DIP-EDL achieves calibration inherently through its probabilistic construction: as $n$ grows, Dirichlet concentration parameters scale proportionally, yielding asymptotically calibrated ID predictions (see Theorem~\ref{thm:asymptotic_consistency_DIP_EDL}), as validated in Tables~\ref{tab:mnist_results},~\ref{tab:cifar_results}, and~\ref{tab:lamost_results}.

    \vspace{-1mm}
\section{Discussion}
\label{sec:conclusion}
    \vspace{-1mm}

We provided a principled statistical interpretation of EDL as amortized VI in a hierarchical Bayesian model, clarifying the regularization parameter's role and exposing its fundamental limitation. Particularly, EDL is unable to faithfully quantify uncertainty, leading to overconfidence on ID and OOD data. To address this, we introduced DIP-EDL, which decouples class prediction from uncertainty by scaling pseudo-counts with marginal covariate density, thus enabling distributional awareness. Theoretically, DIP-EDL achieves asymptotic concentration of the posterior to the true conditional distribution. Empirically, it outperforms established baselines (EDL, R-EDL, Re-EDL, DAEDL, PostNet, MC Dropout) in ID calibration and OOD detection on MNIST, CIFAR-10, and LAMOST.

The broader significance of this work lies in at least three directions. First, it provides a principled statistical foundation for EDL, clarifying its theoretical basis. Second, DIP-EDL generalizes the EDL framework by decoupling classification and uncertainty quantification, training each component independently; moreover, it applies to arbitrary classifier classes, including non-differentiable ones. Third, these gains come without sacrificing the practical efficiency of EDL: a single forward pass at inference time makes DIP-EDL a scalable and modular framework for reliable uncertainty-aware prediction. Future work could refine density estimation and establish finite-sample guarantees.     \vspace{-1mm}

\begin{ack}
    \vspace{-2mm}
    This work was supported by the NSF under Cooperative Agreement 2421782, the Simons Foundation grant MPS-AI-00010515 awarded to the NSF-Simons AI Institute for Cosmic Origins (\href{https://www.cosmicai.org/}{CosmicAI}), and the \href{https://www.bancaditalia.it/chi-siamo/lavorare-bi/borse-di-studio/stringher-mortara-menichella/}{Giorgio Mortara scholarship} by the Bank of Italy. The authors acknowledge the \href{http://www.tacc.utexas.edu}{Texas Advanced Computing Center (TACC)} at The University of Texas at Austin for providing computational resources that have contributed to the research results reported within this paper.
\end{ack}

%%%%%%%%%%%%% REFERENCES %%%%%%%%%%%%%

\bibliography{references}
\bibliographystyle{plainnat}

%%%%%%%%%%%%%%%%% APPENDIX %%%%%%%%%%%%%%%%%
\newpage

\appendix

\section{Reproducibility}
The code to reproduce the results of the numerical experiments is available at: \url{https://github.com/NevenaGligic/DIP-EDL}.

\section{Related Literature.}
\label{sec:appendix_related_literature}

This section contextualizes the proposed DIP-EDL method within the broader landscape of uncertainty quantification with EDL approaches. We first contrast our approach with existing EDL paradigms before summarizing the critical shortcomings identified in current EDL literature. For extensive surveys, we refer the reader to \citet{ulmer2021prior} for EDL-specific developments and \citet{shen2024uncertainty} for a comprehensive overview of existing limitations.

\subsection{Contrasting DIP-EDL with Existing EDL Paradigms}
\label{sec:appendix_contrasting_edl}

Following the taxonomy by \citet{ulmer2021prior}, EDL methods for Dirichlet-based classification are primarily distinguished by two criteria: (i) whether the network parameterizes the prior or the posterior Dirichlet distribution, and (ii) whether the training process incorporates OOD samples to enforce uncertainty calibration.

Regarding the first criterion, \citet{ulmer2021prior} distinguish between \textit{prior} and \textit{posterior} networks. While both share a common optimization objective, balancing a predictive loss (e.g., cross-entropy) with an uncertainty-promoting regularizer (e.g., KL divergence toward a flat Dirichlet), their modeling mechanisms differ. Prior networks \citep{haussmann2019bayesian, tsiligkaridis2021information} parameterize an input-dependent concentration vector that defines a Dirichlet prior and derive the posterior via Bayes' theorem given the observed labels. Conversely, posterior networks \citep{sensoy2018evidential, postnet} bypass the explicit Bayesian update by directly predicting the pseudo-counts that define the concentration parameters of the Dirichlet posterior.

As for the second criterion, the idea is to train a model to output sharp Dirichlet distributions for in-distribution data and flat Dirichlet distributions for OOD data. Therefore, to achieve this kind of uncertainty calibration, the loss function is augmented with an additional regularization term, typically a KL divergence measure between the predicted Dirichlet distribution and a flat Dirichlet distribution, that penalizes confident predictions on OOD samples. In this paradigm, training is agnostic to the network type and can be applied to both prior \citep{malinin2018predictive,malinin2019reverse} and posterior \citep{zhao2019quantifying, sensoy2020uncertainty} network architectures.

While the aforementioned criteria categorize most EDL classification methods, certain approaches based on knowledge distillation \citep{hinton2015distilling} fall outside this taxonomy. These methods transfer uncertainty estimates from a complex teacher model to a more efficient student EDL model, training the latter to mimic the teacher's distributional output. Notable examples include \citet{malinin2019ensemble}, which distills the diversity of an ensemble into a single Dirichlet model, and \citet{fathullah2022self}, which leverages self-distillation from Gaussian stochastic dropout.

Another notable departure from the standard taxonomy is the Fisher Information-based EDL framework proposed by \citet{deng2023uncertainty}. This approach leverages the Fisher information of the Dirichlet distribution to quantify the informativeness of the evidence provided by each observation. By incorporating this metric, the model adjusts its concentration parameters, enhancing its sensitivity to evidence associated with uncertain or underrepresented classes.

\subsection{Shortcomings of Existing EDL Approaches}
\label{sec:appendix_edl_shortcomings}

Despite the growing popularity of EDL, several studies have identified fundamental limitations that challenge its ability to provide reliable uncertainty estimates. These critiques are summarized below.

More specifically, \citet{bengs2022pitfalls} demonstrated that loss minimization is not a theoretically viable framework for learning distributions over probability distributions, such as the Dirichlet distributions employed in EDL. A critical consequence of this result is that the distributional uncertainty in EDL models fails to vanish even in the asymptotic limit of infinite training data. \citet{bengs2023second} further extended this critique, proving that no loss function can successfully incentivize a second-order model to faithfully represent its epistemic uncertainty.

In a more recent theoretical analysis, \citet{shen2024uncertainty} generalized previous critiques by characterizing the optimal distribution targeted by standard EDL models. They demonstrated that this distribution depends fundamentally on a sample-size-independent regularization coefficient, leading to two critical conclusions: (i) EDL models fail to faithfully represent aleatoric uncertainty, as its quantification is governed by an arbitrary hyperparameter rather than solely by the DGP; and (ii) they cannot reliably capture epistemic uncertainty, because the predicted Dirichlet distributions fail to concentrate around the true conditional distribution—even in the asymptotic limit—due to this persistent dependence on the regularization term.

Other works, such as \citet{chen2024r} have also correctly identified the issue of the fundamental dependence of EDL models on a regularization hyperparameter and proposed possible remedies to it. In particular, \citet{chen2024r} introduced R-EDL, treats the regularization hyperparameter as an adjustable hyperparameter, as opposed to a fixed constant, and tunes it with respect to some performance metric (e.g., OOD detection AUROC) over a validation set. While this approach improves empirical performance, it does not address the fundamental theoretical limitations of EDL models identified in prior works.

\section{Proofs and Additional Derivations}
\label{sec:appendix_proofs}

We provide proofs for all propositions and theorems in the main text, beginning with a lemma on the Dirichlet distribution, which is stated without proof, as it is a well-known result.

\begin{lemma}
\label{lemma:dirichlet_expectation}
    \textbf{Expectation of the Dirichlet Distribution.} If $p \sim \mathrm{Dir}(\alpha)$ with $\alpha \in \R_{+}^{K}$, then
    \begin{equation*}
        \E[p(k)] = \frac{\alpha(k)}{\alpha_0}, \quad k = 1, \ldots, K,
    \end{equation*}
    where $\alpha_0 = \sum_{k=1}^{K} \alpha(k)$.
\end{lemma}

\subsection{Proof of Proposition~\ref{prop:icd_model_posterior}}
\label{proof:icd_model_posterior}

\begin{proof}
    By Bayes' theorem, we have that
    \begin{align*}
        \pi_{Y_i}(p_i) & \propto f_{p_{1:n}}(Y_{1:n}) \; \pi(p_{1:n}) \\
        & \propto f_{p_i}(Y_i) \; \pi(p_i) \\
        & \propto \prod_{k=1}^{K} p_i(k)^{\mathbb{I}(Y_i = k)} \; \prod_{k=1}^{K} p_i(k)^{\alpha(k) - 1} \\
        & = \prod_{k=1}^{K} p_i(k)^{\alpha(k) + \mathbb{I}(Y_i = k) - 1}.
    \end{align*}
    We recognize this as the kernel of a $\mathrm{Dir}(\alpha + e_{Y_i})$, so
    \begin{equation*}
        p_i \mid Y_{1:n} \sim \pi_{Y_i} = \mathrm{Dir}(\alpha + e_{Y_i}).
    \end{equation*}
    Since the likelihood and prior factorize across observations, so does the joint posterior:
    \begin{equation*}
        \pi_{Y_{1:n}}(p_{1:n}) \propto f_{p_{1:n}}(Y_{1:n}) \; \pi(p_{1:n}) = \prod_{i=1}^{n} f_{p_i}(Y_i) \; \pi(p_i) \propto \prod_{i=1}^{n} \pi_{Y_i}(p_i).
    \end{equation*}
    As a consequence, the posterior distribution of new class probabilities is equal to the prior:
    \begin{equation*}
        \pi_{Y_{1:n}}(p_{n+1}) \propto \prod_{i=1}^{n} f_{p_i}(Y_i) \; \prod_{i=1}^{n+1} \pi(p_i) \propto \pi(p_{n+1}) = \mathrm{Dir}(\alpha).
    \end{equation*}
    Finally, the posterior predictive distribution follows from Lemma~\ref{lemma:dirichlet_expectation}:
    \begin{align*}
        \Prob(Y_{n+1} = k \mid Y_{1:n}) & = \int \Prob(Y_{n+1} = k \mid p_{n+1}) \; \pi_{Y_{1:n}}(p_{n+1}) \; dp_{n+1} \\
        & = \int p_{n+1}(k) \; \pi(p_{n+1}) \; dp_{n+1} \\
        & = \E_{\pi} [p_{n+1}(k)] \\
        & = \frac{\alpha(k)}{\alpha_0}, \quad k = 1, \ldots, K.
    \end{align*}
    We conclude that
    \begin{equation*}
        Y_{n+1} \mid Y_{1:n} \sim \mathrm{Cat}\!\left( \frac{\alpha}{\alpha_0} \right).
    \end{equation*}
\end{proof}

\subsection{Covariate-Indexed Categorical-Dirichlet Model}
\label{proof:covariate_indexed_cd_model_posterior}

\begin{definition}
\label{def:covariate_indexed_cd_model}
    \textbf{Covariate-Indexed Categorical-Dirichlet Model.} The model is defined as follows:
    \begin{align*}
        Y_i \mid p_i & \ind f_{p_i} = \mathrm{Cat}(p_i), \quad i = 1, \ldots, n, \\
        p_i &= p_{X_i}, \quad X_i \iid P^*_X, \\
        p_j & \iid \pi = \mathrm{Dir}(\alpha), \quad j = 1, \ldots, M_n,
    \end{align*}
    for some $\alpha \in \R_{+}^{K}$, where $M_n = |\{X_{1:n}\}|$ is the number of distinct covariate values.
\end{definition}

This model retains Dirichlet-Categorical conjugacy, enabling closed-form posterior inference.

\begin{proposition}
\label{prop:covariate_indexed_cd_model_posterior}
    Given Definition~\ref{def:covariate_indexed_cd_model}, the posterior and posterior predictive are
    \begin{align*}
        p_{j} \mid X_{1:n}, Y_{1:n} &\sim \mathrm{Dir}(\alpha + c_{j}), \; j = 1, \ldots, M_{n}, \\
        p_{1:M_{n}} \mid X_{1:n}, Y_{1:n} &\sim \prod_{j=1}^{M_{n}} \mathrm{Dir}(\alpha + c_{j}), \\
        Y_{n+1} \mid X_{1:n+1}, Y_{1:n} &\sim \mathrm{Cat}\!\left( \frac{\alpha + c_{x}}{\alpha_0 + S_{x}} \right),
    \end{align*}
    where $c_{j} = \sum_{i=1}^{n} \mathbb{I}(X_{i} = x^{*}_{j}) \, e_{Y_{i}}$ is the label count vector at $x^{*}_{j}$, $x^{*}_{1:M_{n}}$ are the distinct covariate values in $X_{1:n}$, $c_{x} = c_{j}$ if $x = x_{j}$ for some $j$ and $\mathbf{0}_{K}$ otherwise, and $S_{x} = \sum_{k=1}^{K} c_{x}(k)$.
\end{proposition}

\begin{proof}
    By Bayes' theorem, we have that
    \begin{align*}
        \pi_{X_{1:n}, Y_{1:n}}(p_{j}) & \propto f_{p_{1:M_{n}}}(Y_{1:n} \mid X_{1:n}) \; \pi(p_{1:M_{n}} \mid X_{1:n}) \\
        & = \prod_{i : X_{i} = x^{*}_{j}} f_{p_{j}}(Y_{i}) \; \pi(p_{j}) \\
        & \propto \prod_{k=1}^{K} p_{j}(k)^{\sum_{i : X_{i} = x^{*}_{j}} \mathbb{I}(Y_{i} = k)} \; \prod_{k=1}^{K} p_{j}(k)^{\alpha(k) - 1} \\
        & = \prod_{k=1}^{K} p_{j}(k)^{\alpha(k) + c_{j}(k) - 1}.
    \end{align*}
    This is the kernel of a $\mathrm{Dir}(\alpha + c_{j})$, so
    \begin{equation*}
        p_{j} \mid X_{1:n}, Y_{1:n} \sim \pi_{X_{1:n}, Y_{1:n}} = \mathrm{Dir}(\alpha + c_{j}).
    \end{equation*}
    Since the likelihood and prior factorize across unique covariate values, so does the joint posterior:
    \begin{equation*}
        \pi_{X_{1:n}, Y_{1:n}}(p_{1:M_{n}}) \propto f_{p_{1:M_{n}}}(Y_{1:n} \mid X_{1:n}) \; \pi(p_{1:M_{n}} \mid X_{1:n}) \propto \prod_{j=1}^{M_{n}} f_{p_{j}}(Y_{i : X_{i} = x^{*}_{j}}) \; \pi(p_{j}) \propto \prod_{j=1}^{M_{n}} \pi_{X_{1:n}, Y_{1:n}}(p_{j}).
    \end{equation*}
    As a consequence, the posterior distribution of new class probabilities is equal to the prior if the covariate is new and to the corresponding posterior if the covariate has been observed:
    \begin{align*}
        \pi_{X_{1:n+1}, Y_{1:n}}(p_{n+1}) & \propto f_{p_{1:M_{n}}}(Y_{1:n} \mid X_{1:n}) \; \pi(p_{1:M_{n}}, p_{n+1} \mid X_{1:n+1}) \\
        & \propto \begin{cases}
            \pi(p_{n+1}) = \mathrm{Dir}(\alpha), & \text{if } X_{n+1} \notin \{X_{1:n}\}, \\
            \pi_{X_{1:n}, Y_{1:n}}(p_{j}) = \mathrm{Dir}(\alpha + c_{j}), & \text{if } X_{n+1} = x^{*}_{j} \text{ for some } j = 1, \ldots, M_{n}.
        \end{cases}
    \end{align*}
    Finally, the posterior predictive distribution follows from Lemma~\ref{lemma:dirichlet_expectation}:
    \begin{align*}
        \Prob(Y_{n+1} = k \mid X_{1:n+1}, Y_{1:n}) & = \int \Prob(Y_{n+1} = k \mid p_{n+1}) \; \pi_{X_{1:n+1}, Y_{1:n}}(p_{n+1}) \; dp_{n+1} \\
        & = \int p_{n+1}(k) \; \pi_{X_{1:n+1}, Y_{1:n}}(p_{n+1}) \; dp_{n+1} \\
        & = \E_{\pi_{X_{1:n+1}, Y_{1:n}}} [p_{n+1}(k)] \\
        & = \begin{cases}
            \frac{\alpha(k)}{\alpha_0}, & \text{if } X_{n+1} \notin \{X_{1:n}\}, \\
            \frac{\alpha(k) + c_{j}(k)}{\alpha_0 + \sum_{k=1}^{K} c_{j}(k)}, & \text{if } X_{n+1} = x^{*}_{j} \text{ for some } j = 1, \ldots, M_{n}.
        \end{cases}
    \end{align*}
\end{proof}

\subsection{Proof of Proposition~\ref{prop:amortized_vi_icd_predictive}}
\label{proof:amortized_vi_icd_predictive}

\begin{proof}
    By Definition~\ref{def:amortized_vi_icd}, the approximate posterior distribution of new class probabilities is
    \begin{equation*}
        p_{n+1} \mid X_{n+1} \approx q^{\widehat{\phi}_{n}}_{X_{n+1}} = \mathrm{Dir} \left( \alpha + \mathrm{NN}^{\widehat{\phi}_{n}}_{X_{n+1}} \right).
    \end{equation*}
    Therefore, the posterior predictive distribution follows from Lemma~\ref{lemma:dirichlet_expectation}:
    \begin{align*}
        \Prob(Y_{n+1} = k \mid X_{1:n+1}, Y_{1:n}) & = \int \Prob(Y_{n+1} = k \mid p_{n+1}) \; q^{\widehat{\phi}_{n}}_{X_{n+1}}(p_{n+1}) \; dp_{n+1} \\
        & = \int p_{n+1}(k) \; q^{\widehat{\phi}_{n}}_{X_{n+1}}(p_{n+1}) \; dp_{n+1} \\
        & = \E_{q^{\widehat{\phi}_{n}}_{X_{n+1}}} [p_{n+1}(k)] \\
        & = \frac{\alpha(k) + \mathrm{NN}^{\widehat{\phi}_{n}}_{X_{n+1}}(k)}{\alpha_0 + \sum_{k=1}^{K} \mathrm{NN}^{\widehat{\phi}_{n}}_{X_{n+1}}(k)}.
    \end{align*}
\end{proof}

\subsection{Proof of Proposition~\ref{prop:KL_expansion_icd_model}}
\label{proof:KL_expansion_icd_model}

\begin{proof}
    Since both the variational distribution and the posterior factorize across observations, we have
    \begin{align*}
        \text{KL} \left(q^{\phi}_{X_{1:n}} \; || \; \pi_{Y_{1:n}}\right) 
        &= \E_{q^{\phi}_{X_{1:n}}} \left[ \log \frac{\prod_{i=1}^{n} q^{\phi}_{X_{i}}(p_{i})}{\prod_{i=1}^{n} \pi_{Y_{i}}(p_{i})} \right] \\
        &= \sum_{i=1}^{n} \E_{q^{\phi}_{X_{i}}} \left[ \log \frac{q^{\phi}_{X_{i}}(p_{i})}{\pi_{Y_{i}}(p_{i})} \right] \\
        &\propto \sum_{i=1}^{n} \E_{q^{\phi}_{X_{i}}} \left[ \log \frac{q^{\phi}_{X_{i}}(p_{i})}{f_{p_{i}}(Y_{i}) \; \pi(p_{i})} \right] \\
        &= \sum_{i=1}^{n} - \E_{q^{\phi}_{X_{i}}} \left[ \log f_{p_{i}} (Y_{i}) \right] + \E_{q^{\phi}_{X_{i}}} \left[ \log \frac{q^{\phi}_{X_{i}}(p_{i})}{\pi(p_{i})} \right] \\
        &= \sum_{i=1}^{n} - \E_{q^{\phi}_{X_{i}}} \left[ \log f_{p_{i}} (Y_{i}) \right] + \text{KL}\left( q^{\phi}_{X_{i}} \; || \; \pi \right).
    \end{align*}
\end{proof}

\subsection{Proof of Proposition~\ref{prop:tempered_icd_model_posterior}}
\label{proof:tempered_icd_model_posterior}

\begin{proof}
    The proof follows the same steps as in the proof of Proposition~\ref{prop:icd_model_posterior}, with the only difference being the tempered likelihood. Specifically, by Bayes' theorem, the posterior distribution of the class probabilities for each observation satisfies
    \begin{align*}
        \pi^{\nu}_{Y_i}(p_i) & \propto f_{p_{1:n}}(Y_{1:n})^{\nu} \; \pi(p_{1:n}) \\
        & \propto f_{p_i}(Y_i)^{\nu} \; \pi(p_i) \\
        & \propto \left( \prod_{k=1}^{K} p_i(k)^{\mathbb{I}(Y_i = k)} \right)^{\nu} \; \prod_{k=1}^{K} p_i(k)^{\alpha(k) - 1} \\
        & = \prod_{k=1}^{K} p_i(k)^{\alpha(k) + \nu \; \mathbb{I}(Y_i = k) - 1}.
    \end{align*}
    This is the kernel of a $\mathrm{Dir}(\alpha + \nu \; e_{Y_i})$, so
    \begin{equation*}
        p_i \mid Y_{1:n} \sim \pi^{\nu}_{Y_i} = \mathrm{Dir}(\alpha + \nu \; e_{Y_i}).
    \end{equation*}
    Since the likelihood and prior factorize across observations, so does the joint posterior:
    \begin{equation*}
        \pi^{\nu}_{Y_{1:n}}(p_{1:n}) \propto f_{p_{1:n}}(Y_{1:n})^{\nu} \; \pi(p_{1:n}) \propto \prod_{i=1}^{n} f_{p_i}(Y_i)^{\nu} \; \pi(p_i) \propto \prod_{i=1}^{n} \pi^{\nu}_{Y_i}(p_i).
    \end{equation*}
    As a consequence, the posterior distribution of new class probabilities is equal to the prior:
    \begin{equation*}
        \pi^{\nu}_{Y_{1:n}}(p_{n+1}) \propto \prod_{i=1}^{n} f_{p_i}(Y_i)^{\nu} \; \prod_{i=1}^{n+1} \pi(p_i) \propto \pi(p_{n+1}) = \mathrm{Dir}(\alpha).
    \end{equation*}
    Finally, the posterior predictive distribution follows from Lemma~\ref{lemma:dirichlet_expectation}:
    \begin{align*}
        \Prob(Y_{n+1} = k \mid Y_{1:n}) & = \int \Prob(Y_{n+1} = k \mid p_{n+1}) \; \pi^{\nu}_{Y_{1:n}}(p_{n+1}) \; dp_{n+1} \\
        & = \int p_{n+1}(k) \; \pi(p_{n+1}) \; dp_{n+1} \\
        & = \E_{\pi} [p_{n+1}(k)] \\
        & = \frac{\alpha(k)}{\alpha_0}.
    \end{align*}
\end{proof}

\subsection{Proof of Theorem~\ref{thm:equivalence_edl_amortized_vi}}
\label{proof:equivalence_edl_amortized_vi}

\begin{proof}
    By the same steps as in the proof of Proposition~\ref{prop:KL_expansion_icd_model}:
    \begin{align*}
        \text{KL} \left(q^{\phi}_{X_{1:n}} \; || \; \pi^{\nu}_{Y_{1:n}}\right) 
        &= \E_{q^{\phi}_{X_{1:n}}} \left[ \log \frac{\prod_{i=1}^{n} q^{\phi}_{X_{i}}(p_{i})}{\prod_{i=1}^{n} \pi^{\nu}_{Y_{i}}(p_{i})} \right] \\
        &= \sum_{i=1}^{n} \E_{q^{\phi}_{X_{i}}} \left[ \log \frac{q^{\phi}_{X_{i}}(p_{i})}{\pi^{\nu}_{Y_{i}}(p_{i})} \right] \\
        &\propto \sum_{i=1}^{n} \E_{q^{\phi}_{X_{i}}} \left[ \log \frac{q^{\phi}_{X_{i}}(p_{i})}{f_{p_{i}}(Y_{i})^{\nu} \; \pi(p_{i})} \right] \\
        &= \sum_{i=1}^{n} - \nu \; \E_{q^{\phi}_{X_{i}}} \left[ \log f_{p_{i}} (Y_{i}) \right] + \E_{q^{\phi}_{X_{i}}} \left[ \log \frac{q^{\phi}_{X_{i}}(p_{i})}{\pi(p_{i})} \right] \\
        &= \nu \; \sum_{i=1}^{n} - \E_{q^{\phi}_{X_{i}}} \left[ \log f_{p_{i}} (Y_{i}) \right] + \frac{1}{\nu} \; \text{KL}\left( q^{\phi}_{X_{i}} \; || \; \pi \right) \\
        & = \nu \; \mathcal{L}^{\nu} (\phi; Y_{1:n}, X_{1:n}).
    \end{align*}
    Therefore, since $\nu$ is a positive constant, we have that
    \begin{equation*}
        \widehat{\phi}_{n} = \argmin_{\phi} \left\{ \text{KL} \left(q^{\phi}_{X_{1:n}} \; || \; \pi^{\nu}_{Y_{1:n}}\right) \right\} = \argmin_{\phi} \left\{ \mathcal{L}^{\nu} (\phi; Y_{1:n}, X_{1:n}) \right\}.
    \end{equation*}
\end{proof}

\subsection{Proof of Proposition~\ref{prop:perfect_interpolation_condition}}
\label{proof:perfect_interpolation_condition}

Let us begin with the following lemma on the properties of the KL divergence, which is stated without proof, as it is a standard result.

\begin{lemma}
\label{lemma:kl_divergence_non_negative}
    \textbf{Properties of the Kullback-Leibler Divergence.}  
    Let $f$ and $g$ be two probability density functions defined on the same probability space. Then, the KL divergence from $f$ to $g$ is non-negative:
    \begin{equation*}
        \text{KL} (f \; || \; g) \geq 0,
    \end{equation*}
    with equality if and only if $f$ and $g$ are equal almost surely.
\end{lemma}

\begin{proof}
    By Theorem~\ref{thm:equivalence_edl_amortized_vi}, we have that
    \begin{equation*}
        \min_{\phi \in \Phi} \left\{ \mathcal{L}^{\nu} (\phi; Y_{1:n}, X_{1:n}) \right\} = \frac{1}{\nu} \; \min_{\phi \in \Phi} \left\{ \text{KL} \left(q^{\phi}_{X_{1:n}} \; || \; \pi^{\nu}_{Y_{1:n}}\right) \right\}.
    \end{equation*}
    If $\Phi$ is sufficiently expressive and the optimization procedure achieves the minimizer of the KL divergence, then by Lemma~\ref{lemma:kl_divergence_non_negative}, we have that
    \begin{equation*}
        \text{KL} \left(q^{\widehat{\phi}_{n}}_{X_{1:n}} \; || \; \pi^{\nu}_{Y_{1:n}}\right) = 0,
    \end{equation*}
    which holds if and only if
    \begin{equation*}
        q^{\widehat{\phi}_{n}}_{X_{1:n}} = \pi^{\nu}_{Y_{1:n}}.
    \end{equation*}
    In turn, since both distributions factorize across observations, this is equivalent to requiring that
    \begin{equation*}
        q^{\widehat{\phi}_{n}}_{X_{i}} = \pi^{\nu}_{Y_{i}}, \quad \text{for } i = 1, \ldots, n.
    \end{equation*}
    By Definitions~\ref{def:amortized_vi_tempered_icd} and~\ref{def:tempered_icd_model}, this is equivalent to
    \begin{equation*}
        \mathrm{Dir} (\alpha + \mathrm{NN}^{\widehat{\phi}_{n}}_{X_{i}}) = \mathrm{Dir}(\alpha + \nu \; e_{Y_{i}}), \quad \text{for } i = 1, \ldots, n,
    \end{equation*}
    which holds if and only if
    \begin{equation*}
        \mathrm{NN}^{\widehat{\phi}_{n}}_{X_{i}} = \nu \; e_{Y_{i}}, \quad \text{for } i = 1, \ldots, n.
    \end{equation*}
\end{proof}

\subsection{Proof of Proposition~\ref{prop:vacuity_perfect_interpolation}}
\label{proof:vacuity_perfect_interpolation}

By Definition~\ref{def:vacuity}, the vacuity of the approximate posterior predictive distribution is as follows:
\begin{align*}
    u \left( \alpha + \mathrm{NN}^{\widehat{\phi}_{n}}_{X_{n+1}} \right) & = \frac{K}{\alpha_0 + \sum_{k=1}^{K} \mathrm{NN}^{\widehat{\phi}_{n}}_{X_{n+1}}(k)} \\
    & = \frac{K}{\alpha_0 + \sum_{k=1}^{K} \nu \; e_{Y_{n+1}}(k)} \\
    & = \frac{K}{\alpha_0 + \nu},
\end{align*}
where we used the perfect interpolation condition from Proposition~\ref{prop:perfect_interpolation_condition} in the second step.

\subsection{Proof of Theorem~\ref{thm:equivalence_edl_erm}}
\label{proof:equivalence_edl_erm}

We can rewrite the empirical risk as follows:
\begin{align*}
    \widehat{\mathcal{R}}^{\nu}_{n} (\phi) & = \frac{1}{n} \; \sum_{i=1}^{n} \text{KL} \left( q^{\phi}_{X_{i}} \; || \; \pi^{\nu}_{Y_{i}} \right) \\
    & = \frac{1}{n} \; \text{KL} \left( q^{\phi}_{X_{1:n}} \; || \; \pi^{\nu}_{Y_{1:n}} \right) \\
    & = \frac{\nu}{n} \; \mathcal{L}^{\frac{1}{\nu}} (\phi; Y_{1:n}, X_{1:n}),
\end{align*}
where in the last step we skipped the same steps as in the proof of Theorem~\ref{thm:equivalence_edl_amortized_vi}.

\subsection{Proof of Theorem~\ref{thm:oracle_variational_distribution}}
\label{proof:oracle_variational_distribution}

By the law or iterated expectations, we can rewrite the population risk as follows:
\begin{align*}
    \mathcal{R}^{\nu} (\phi) & = \E_{P^{*}_{X}} \left[ \E_{P^{*}_{Y|X}} \left[ \text{KL}\left( q^{\phi}_{X} \; || \; \pi^{\nu}_{Y} \right) \right] \right] \\
    & = \E_{P^{*}_{X}} \left[ \E_{P^{*}_{Y|X}} \left[ \E_{q^{\phi}_{X}} \left[ \log q^{\phi}_{X}(p) - \log \pi^{\nu}_{Y}(p) \right] \right] \right] \\
    & = \E_{P^{*}_{X}} \left[ \E_{q^{\phi}_{X}} \left[ \log q^{\phi}_{X}(p) - \E_{P^{*}_{Y|X}} \left[ \log \pi^{\nu}_{Y}(p) \right] \right] \right],
\end{align*}
where we interchanged the order of expectations in the last step, since the first term is independent of $Y$, and the second term involves an expectation over $P^{*}_{Y|X}$, which is a finite sum and can therefore be freely reordered with the expectation over the variational distribution.

Then, for each fixed $X$, by Gibb's variational principle, the population risk is minimized by choosing the parameter $\phi^{*}$ such that the variational distribution satisfies
\begin{align*}
    q^{\phi^{*}}_{X} & \propto \exp \left( \E_{P^{*}_{Y|X}} \left[ \log \pi^{\nu}_{Y} (p) \right] \right) \\
    & = \exp \left( \E_{P^{*}_{Y|X}} \left[ \log \mathrm{Dir} (\alpha + \nu \; e_{Y}) \right] \right) \\
    & = \exp \left( \E_{P^{*}_{Y|X}} \left[ \sum_{k=1}^{K} (\alpha (k) + \nu \; e_{Y}(k) - 1) \; \log p (k) - \log B(\alpha + \nu \; e_{Y}) \right] \right) \\
    & \propto \exp \left( \sum_{k=1}^{K} \left( \alpha (k) + \nu \; P^{*}_{Y|X}(k \mid X ) - 1  \right) \; \log p (k) \right) \\
    & = \prod_{k=1}^{K} p (k)^{\alpha (k) + \nu \; P^{*}_{Y|X}(k \mid X ) - 1},
\end{align*}
where $B$ is the multivariate Beta function. We recognize this as the kernel of a $\mathrm{Dir}(\alpha + \nu \; P^{*}_{Y|X})$, so
\begin{equation*}
    q^{\phi^{*}}_{X} = \mathrm{Dir} \left( \alpha + \nu \; P^{*}_{Y|X} \right).
\end{equation*}

\subsection{Proof of Theorem~\ref{thm:asymptotic_consistency_DIP_EDL}}
\label{proof:asymptotic_consistency_DIP_EDL}

Under Definition~\ref{def:DIP_EDL}, the approximate posterior mean is
\begin{align*}
    \E [p_{i} \mid X_{1:n}, Y_{1:n}] & \approx \frac{\alpha + n \, \mathrm{DE}^{\widehat{\psi}_{n}}_{X_{i}} \, \mathrm{NN}^{\widehat{\psi}_{n}}_{X_{i}}}{\alpha_{0} + n \, \mathrm{DE}^{\widehat{\psi}_{n}}_{X_{i}}} \\
    & = \frac{\frac{\alpha}{n} + \mathrm{DE}^{\widehat{\psi}_{n}}_{X_{i}} \, \mathrm{NN}^{\widehat{\psi}_{n}}_{X_{i}}}{\frac{\alpha_{0}}{n} + \mathrm{DE}^{\widehat{\psi}_{n}}_{X_{i}}} \\
    & \overset{p}{\to} \frac{P^{*}_{X} (X_{i}) \, P^{*}_{Y \mid X} (\cdot \mid X_{i})}{P^{*}_{X} (X_{i})} \\
    & = P^{*}_{Y \mid X} (\cdot \mid X_{i}),
\end{align*}
where the convergence in probability follows from the consistency of the two neural networks, as stated in Theorem~\ref{thm:asymptotic_consistency_DIP_EDL} and the continuous mapping theorem. Similarly, the approximate posterior variance of the class probabilities is
\begin{align*}
    \text{Var} [p_{i} (k) \mid X_{1:n}, Y_{1:n}] & \approx \frac{ \frac{ \alpha + n \, \mathrm{DE}^{\widehat{\psi}_{n}}_{X_{i}} \, \mathrm{NN}^{\widehat{\psi}_{n}}_{X_{i}} (k) }{ \alpha_{0} + n \, \mathrm{DE}^{\widehat{\psi}_{n}}_{X_{i}}} \left( 1 - \frac{ \alpha + n \, \mathrm{DE}^{\widehat{\psi}_{n}}_{X_{i}} \, \mathrm{NN}^{\widehat{\psi}_{n}}_{X_{i}} (k) }{ \alpha_{0} + n \, \mathrm{DE}^{\widehat{\psi}_{n}}_{X_{i}}} \right) } { \alpha_{0} + n \, \mathrm{DE}^{\widehat{\psi}_{n}}_{X_{i}} + 1} \\
    & = \frac{ \frac{ \frac{\alpha}{n} + \mathrm{DE}^{\widehat{\psi}_{n}}_{X_{i}} \, \mathrm{NN}^{\widehat{\psi}_{n}}_{X_{i}} (k) }{ \frac{\alpha_{0}}{n} + \mathrm{DE}^{\widehat{\psi}_{n}}_{X_{i}}} \left( 1 - \frac{ \frac{\alpha}{n} + \mathrm{DE}^{\widehat{\psi}_{n}}_{X_{i}} \, \mathrm{NN}^{\widehat{\psi}_{n}}_{X_{i}} (k) }{ \frac{\alpha_{0}}{n} + \mathrm{DE}^{\widehat{\psi}_{n}}_{X_{i}}} \right) } { \alpha_{0} + n \, \mathrm{DE}^{\widehat{\psi}_{n}}_{X_{i}} + 1} \\
    & \overset{p}{\approx} \frac{ \frac{ P^{*}_{X} (X_{i}) \, P^{*}_{Y \mid X} (k \mid X_{i}) }{ P^{*}_{X} (X_{i})} \left( 1 - \frac{ P^{*}_{X} (X_{i}) \, P^{*}_{Y \mid X} (k \mid X_{i}) }{ P^{*}_{X} (X_{i})} \right) } { \alpha_{0} + n \, P^{*}_{X} (X_{i}) + 1} \\
    & = \frac{ P^{*}_{Y \mid X} (k \mid X_{i}) \left( 1 - P^{*}_{Y \mid X} (k \mid X_{i}) \right) } { \alpha_{0} + n \, P^{*}_{X} (X_{i}) + 1} \\
    & \overset{n}{\to} 0,
\end{align*}
where the convergence in probability again follows from the consistency of the two neural networks and the continuous mapping theorem, while the final step follows from the fact that the denominator diverges to infinity as $n$ goes to infinity.

Therefore, by Chebyshev's inequality, we have that
\begin{align*}
    \Prob ( | p_{i}(k) - \E [p_{i}(k) \mid X_{1:n}, Y_{1:n}] | > \varepsilon \mid  X_{1:n}, Y_{1:n} ) \leq \frac{\text{Var} [p_{i} (k) \mid X_{1:n}, Y_{1:n}]}{\varepsilon^2} \overset{n}{\to} 0.
\end{align*}
Hence, we have that
\begin{align*}
    \Prob ( \| p_{i} - \E [p_{i} \mid X_{1:n}, Y_{1:n}] \|_{1} > \varepsilon \mid  X_{1:n}, Y_{1:n} ) & = \Prob \left( \sum^{K}_{k = 1} | p_{i}(k) - \E [p_{i}(k) \mid X_{1:n}, Y_{1:n}] | > \varepsilon \mid  X_{1:n}, Y_{1:n} \right) \\
    & \leq \sum^{K}_{k = 1} \Prob \left( | p_{i}(k) - \E [p_{i}(k) \mid X_{1:n}, Y_{1:n}] | > \frac{\varepsilon}{K} \mid  X_{1:n}, Y_{1:n} \right) \\
    & \overset{n}{\to} 0.
\end{align*}
Finally, by the triangular inequality we have that
\begin{align*}
    \| p_{i} - P^{*}_{Y \mid X} ( \cdot \mid X_{i}) \|_{1} &  \leq \| p_{i} - \E [p_{i} \mid X_{1:n}, Y_{1:n}] \|_{1} + \| \E [p_{i} \mid X_{1:n}, Y_{1:n}] - P^{*}_{Y \mid X} ( \cdot \mid X_{i}) \|_{1} \\
    & \overset{n}{\to} 0
\end{align*}
Therefore, $p_{i} \overset{L_{1}}{\to} P^{*}_{Y \mid X} ( \cdot \mid X_{i}) \mid X_{1:n}, Y_{1:n} $, which implies that $p_{i} \overset{p}{\to} P^{*}_{Y \mid X} ( \cdot \mid X_{i}) \mid X_{1:n}, Y_{1:n} $.

%One may be tempted to choose the temperature by treating as parameter of the loss function and thus minimizing $L$ with respect to both $\phi$ and $\nu$. However, this approach is not viable, since it easy to see that $L (\nu)$ has a unique global minimum at $\nu = 0$.

%First, observe that trivially we have
%\begin{equation}
%    L(0) = \E 
%\end{equation}

%\textcolor{red}{SKETCH OF THE PROOF:
%\begin{itemize}
%    \item $L(\nu) \geq 0$ and $L(0) = KL(Dirichlet(\alpha) || Dirichlet(\alpha)) = 0$. Therefore $L(\nu) \geq L(0)$ for all $\nu \geq 0$.
%    \item Since the loss function is minimized by $q^{*}_{X} (\nu) = Dirichlet(\alpha + \nu \; P^{*}_{Y|X})$, we use that in the following.
%    \item If $P^{*}_{Y|X}$ is degenerate, ie there exists a function $k : \R^{d} \to \{1, \ldots, K\}$ such that $P^{*}_{Y|X = x} = e_{k(x)} \; \forall x \in \R^{d}$ almost surely, then $KL (Dirichlet(\alpha + \nu q^{*}_{X}) || Dirichlet(\alpha + \nu e_{Y})) = 0$ if and only if $\nu = 1$.
%    \item However, if $P^{*}_{Y|X}$ is not degenerate, ie there exists $x \in \R^{d}$ such that $P^{*}_{Y|X = x}$ has at least two non-zero elements, then $KL (Dirichlet(\alpha + \nu q^{*}_{X}) || Dirichlet(\alpha + \nu e_{Y})) = 0$ if and only if $\nu P^{*}_{Y|X = x} = e_{k}$ for some $k$, which is impossible. Therefore, in this case $L(\nu) > 0$ for all $\nu > 0$.
%    \item This proves that for non-degenerate $P^{*}_{Y|X}$, $L(\nu)$ has a unique global minimum at $\nu = 0$.
%\end{itemize}
%}

\section{Detailed Experimental Configuration}
\label{sec:appendix_experiments}

This appendix provides the precise model architectures, training hyperparameters, and data processing steps used in our experiments to ensure reproducibility. All experiments were implemented in PyTorch~2.6~\cite{pytorch} with CUDA~12.4. Experiments were run on a high-performance computing cluster (TACC Vista \cite{vista}) equipped with NVIDIA H200 GPUs (96\,GB HBM3) paired with NVIDIA Grace Arm CPUs (72 cores, 116\,GB DDR5).

\subsection{MNIST Configuration for DIP-EDL}

For the MNIST dataset \cite{mnist}, we employ an architecture operating in pixel space. The density estimator is implemented using the \texttt{nflows} library \cite{nflows}.

\textbf{Backbone Architecture (LeNet-5).} We utilize a standard LeNet-style Convolutional Neural Network \cite{LeNet}.

\begin{table}[htbp]
\centering
\caption{LeNet-5 implementation details for MNIST.}
\label{tab:mnist_arch}
\begin{tabular}{ll}
\hline
\textbf{Layer Type} & \textbf{Specifications} \\ \hline
Conv2d & In: 1, Out: 20, Kernel: $5 \times 5$, Stride: 1, Pad: 2, ReLU \\
MaxPool2d & Kernel: $2 \times 2$, Stride: 2 \\
Conv2d & In: 20, Out: 50, Kernel: $5 \times 5$, Stride: 1, Pad: 0, ReLU \\
MaxPool2d & Kernel: $2 \times 2$, Stride: 2 \\
Flatten & Output Dimension: 1250 ($50 \times 5 \times 5$) \\
Linear & In: 1250, Out: 500, ReLU \\
Dropout & $p=0.5$ \\
Linear (Head) & In: 500, Out: 10 (No activation) \\ \hline
\addlinespace
\end{tabular}
\end{table}

\textbf{Density Estimator (MAF).} We employ a Masked Autoregressive Flow (MAF) \cite{MAF} to model the data density directly on flattened $28 \times 28$ pixel inputs. The flow consists of 10 autoregressive transforms, each implemented as a MADE network with 20 residual blocks and 1024 hidden units per block. We apply Batch Normalization between transforms to improve convergence. 

\textbf{Data Augmentation \& Preprocessing.}
For the density estimator, we adhere to the preprocessing protocol established in \cite{MAF}.
\begin{enumerate}
    \item \textbf{Dequantization:} We add uniform noise $u \sim U(0, 1)$ to the discrete pixel values $x \in \{0, \dots, 255\}$ to convert them into continuous variables: $x' = (x + u) / 256$.
    \item \textbf{Logit Transformation:} To map the bounded data $x' \in [0, 1]$ to the unconstrained real space $\mathbb{R}$ required by the flow, we apply a logit transformation: 
    \[ z = \text{logit}(\lambda + (1 - 2\lambda)x') \]
    where $\lambda = 10^{-6}$ is a regularization parameter to prevent numerical instability at the boundaries.
\end{enumerate}
No standard mean/variance normalization or geometric augmentations (e.g., rotation) were applied to the MNIST data.

\textbf{Training Hyperparameters.}  The training hyperparameters used in this work are listed in Table~\ref{tab:mnist_hyper}.

\begin{table}[htbp]
\centering
\caption{Training Hyperparameters for MNIST}
\label{tab:mnist_hyper}
\begin{tabular}{lll}
\hline
\textbf{Parameter} & \textbf{Backbone (EDL)} & \textbf{Density Estimator (MAF)} \\ \hline
Optimizer & Adam & Adam \\
Batch Size & 128 & 128 \\
Epochs & 50 & 50 \\
Learning Rate & $1 \times 10^{-3}$ & $1 \times 10^{-4}$ \\
Weight Decay & $5 \times 10^{-3}$ & $1 \times 10^{-5}$ \\
Scheduler & None & StepLR (Step: 5, Gamma: 0.5) \\
Annealing Steps & 10 Epochs & N/A \\ \hline
\addlinespace
\end{tabular}
\end{table}

\subsection{CIFAR-10 Configuration for DIP-EDL}

For CIFAR-10 dataset \cite{cifar10}, we utilize a WideResNet-28-10 backbone \cite{zagoruyko2016wide} and perform density estimation in the learned feature space.

% \textbf{Backbone Architecture (Modified ResNet-18).} We adapt the ResNet-18 architecture to handle low-resolution $32 \times 32$ inputs. Specifically, the initial $7 \times 7$ convolution (stride 2) is replaced with a $3 \times 3$ convolution (stride 1), and the first max-pooling layer is removed to preserve spatial resolution. To enforce feature compactness, we apply Spectral Normalization (SN) to all convolutional layers in the backbone.

\textbf{Backbone Architecture.} For a classifier, we use WideResNet-28-10 with pre-activation residual blocks (BatchNorm $\to$ ReLU $\to$ Conv) with widening factor 10, yielding 640-dimensional output embeddings. For training, we use SGD (lr $= 0.1$, 
momentum $= 0.9$, Nesterov, weight decay $= 5{\times}10^{-4}$, cosine annealing 
over 100 epochs). Spectral Normalization is applied to all convolutional and linear layers, which regularizes the Lipschitz constant of the network and produces well-behaved feature representations suitable for post-hoc density estimation.

% \textbf{Density Estimator (GDA).} We employ Gaussian Discriminant Analysis (GDA) on the 512-dimensional feature embeddings, adopting the official implementation from DAEDL \cite{DAEDL}. We perform class-conditional multivariate Gaussian fitting ($C=10$) in a post-hoc manner. After training the backbone, we freeze the weights and extract features for all ID training samples to compute the empirical mean and covariance matrix for each class. We utilize class-specific full covariance matrices (Quadratic Discriminant Analysis) to maximize density estimation precision.

\textbf{Density Estimator.} GDA is fitted post-hoc on the frozen 640-dimensional embeddings using a clean, non-augmented pass over the training set, with full per-class covariance matrices. The resulting log-likelihoods are z-scored using training-set statistics before being used for further evaluation. GDA is fitted on the non-augmented view of the data, as augmented features degrade density estimates. The training set size is $N = 40{,}000$, reflecting the 80/20 train/validation split. 

\textbf{Data Augmentation \& Preprocessing.} We apply standard data augmentations during the classifier training. These include random cropping to $32 \times 32$ with padding of 4, random horizontal flipping with a probability of $0.5$, and random rotation of $\pm 15$ degrees. All images are normalized using the standard CIFAR-10 mean and standard deviation ($\mu=(0.4914, 0.4822, 0.4465)$, $\sigma=(0.2023, 0.1994, 0.2010)$).

\textbf{Training Hyperparameters (CIFAR-10).} The training hyperparameters used in this work are listed in Table~\ref{tab:cifar_hyper}.

\begin{table}[htbp]
\centering
\caption{Training Hyperparameters for CIFAR-10}
\label{tab:cifar_hyper}
\begin{tabular}{ll}
\hline
\textbf{Parameter} & \textbf{Value} \\ \hline
Optimizer          & SGD (momentum $= 0.9$, Nesterov) \\
Learning Rate      & $0.1$ \\
Weight Decay       & $5 \times 10^{-4}$ \\
Batch Size         & 128 \\
Epochs             & 100 \\
Dropout Rate       & $0.3$ \\
LR Scheduler       & Cosine Annealing ($T_{\max} = 100$) \\ \hline
\addlinespace
\end{tabular}
\end{table}

\subsection{LAMOST Configuration for DIP-EDL}

\textbf{Dataset.} We use a subset of the LAMOST Data Release 9 (DR9) dataset, available at \url{https://www.lamost.org/dr9/}. The subset was created by \url{https://github.com/superdreamliner/LAMOST-Spectra-Classifier} (\cite{lamost_spectra_classifier}) and it can be downloaded from \url{https://www.dropbox.com/scl/fi/tp81mfopdqbep50vwhhhb/spectra_training_data.tar}. The dataset is created by drawing a total of 100{,}000 good-quality spectra (signal-to-noise ratio $> 10$) from DR9, spanning three stellar object classes: galaxy, quasar, and star. All spectra are evenly interpolated over the wavelength range 3{,}900--9{,}000\,\AA\ with 3{,}000 data points per spectrum. 

In our experiments, we use Galaxy and Quasar as the two in-distribution classes and hold out Star spectra as the OOD set. The data is partitioned into 60/20/20 train/validation/test splits prior to any feature extraction.

\textbf{Data preprocessing.} Following the Github repository which created this dataset (\cite{lamost_spectra_classifier}), raw flux arrays undergo energy normalization followed by Savitzky--Golay smoothing (window length 10, polynomial order 3). A three-stage feature extraction is then applied:
\begin{enumerate}
    \item \textbf{PCA}: 500 principal components fitted on the training spectra.
    \item \textbf{Spectral line features}: Mean flux within a $\pm3$\,\AA\ window around six known emission/absorption lines (H$\alpha$\,6563, H$\beta$\,4861, H$\gamma$\,4340, H$\delta$\,4102, Mg\,b\,5184, Na\,D\,5890; wavelengths in \AA), yielding 6 features.
    \item \textbf{Statistical block features}: Each spectrum is divided into 50 equal-length blocks. Mean, variance, argmax, argmin, maximum, and minimum are computed per block, yielding 300 features.
\end{enumerate}
The concatenated feature vector has dimension $500 + 6 + 300 = 806$, fed as a
single-channel 1-D signal $(1, 806)$ to the backbone.

\textbf{Backbone Architecture (Multi-Branch 1D CNN).} We use the architecture by
\cite{lamost_spectra_classifier}, a multi-branch 1D convolutional network designed for
stellar spectra classification.

\begin{table}[htbp]
\centering
\caption{Multi-branch 1D CNN architecture for LAMOST.}
\label{tab:lamost_arch}
\resizebox{\textwidth}{!}{
\begin{tabular}{ll}
\hline
\textbf{Component} & \textbf{Specifications} \\ \hline
Parallel branches        & 5 branches, kernel sizes $k \in \{3, 5, 7, 9, 11\}$ \\
Per-branch block ($\times$3) & Conv1d (32 filters, pad $\lfloor k/2 \rfloor$) $\to$ BatchNorm1d $\to$ ReLU $\to$ MaxPool1d (3, stride 3) \\
Merge                    & Concatenate and flatten all branch outputs \\
Dense 1                  & Linear $\to$ BatchNorm1d $\to$ ReLU, Out: 128 \\
Dense 2                  & Linear $\to$ BatchNorm1d $\to$ ReLU, Out: 64 \\
Dense 3                  & Linear $\to$ BatchNorm1d $\to$ ReLU, Out: 32 \\
Head                     & Linear, Out: 2 (no activation) \\ \hline
\end{tabular}
}
\end{table}

\textbf{Density Estimator (GDA).} GDA is fitted post-hoc on the frozen 32-dimensional
penultimate features (output of Dense 3) using the training split, with full per-class
covariance matrices. Log-likelihoods are z-scored (see Section~\ref{sec:zscore}) using training-set statistics before
use in evaluation.

\textbf{Training Hyperparameters.} The training hyperparameters used in this work are listed in Table~\ref{tab:lamost_hyper}.

\begin{table}[htbp]
\centering
\caption{Training Hyperparameters for LAMOST}
\label{tab:lamost_hyper}
\begin{tabular}{ll}
\hline
\textbf{Parameter} & \textbf{Value} \\ \hline
Optimizer     & Adam \\
Learning Rate & $1 \times 10^{-3}$ \\
Batch Size    & 128 \\
Epochs        & 50 \\
Weight Decay  & None \\
LR Scheduler  & None \\ \hline
\end{tabular}
\end{table}

\subsection{Addressing underflow in DIP-EDL due to log-likelihoods with large negative magnitudes.} \label{sec:zscore}
To ensure numerical stability, we address the scale of the output log-likelihoods. Raw log-densities in high-dimensional pixel space can take on large negative magnitudes (e.g., $<-1000$), leading to numerical underflow when exponentiated and converted to evidence. We mitigate this by Z-score normalizing the log-likelihoods using the mean and standard deviation statistics calculated on the In-Distribution (ID) training set.

\subsection{Baseline Implementation Details}
To ensure a rigorous and fair comparison, we utilize the official open-source implementations for all competing methods. For EDL \cite{sensoy2018evidential}, MC Dropout \cite{gal2016dropout}, and Deep Ensembles \cite{lakshminarayanan2017simple}, we align the backbone with DIP-EDL (LeNet-5 for MNIST, WideResNet-28-10 for CIFAR-10, and the multi-branch 1D CNN for LAMOST) to isolate the contribution of the uncertainty quantification mechanism. DAEDL \cite{DAEDL}, R-EDL \cite{REDL}, Re-EDL \cite{chen2025revisiting}, and PostNet \cite{postnet} use their own backbone 
architectures as specified in their respective implementations for MNIST and CIFAR-10, while for LAMOST, where no official spectral architecture exists, DAEDL is adapted with a 1D convolutional backbone following its own architectural design, and R-EDL, Re-EDL, and PostNet are adapted to use linear backbones.
The optimization hyperparameters for all baselines are detailed in Tables~\ref{tab:mnist_baselines}, \ref{tab:cifar_baselines}, and~\ref{tab:lamost_baselines}.

\begin{table}[htbp]
\centering
\caption{Baseline Training Hyperparameters (MNIST)}
\label{tab:mnist_baselines}
\resizebox{\textwidth}{!}{
\begin{tabular}{lccccc}
\hline
\textbf{Method} & \textbf{Epochs} & \textbf{Batch Size} & \textbf{LR} & \textbf{Optimizer} & \textbf{Specific Parameters} \\ \hline
EDL            & 50  & 128 & $1 \times 10^{-3}$ & Adam & Weight decay $= 5\times10^{-3}$, KL annealing $= 10$ epochs \\
R-EDL          & 60  & 64  & $1 \times 10^{-3}$ & Adam & $\lambda_1=1.0,\ \lambda_2=0.1$ \\
Re-EDL         & 60  & 64  & $1 \times 10^{-3}$ & Adam & $\lambda_1=1.0,\ \lambda_2=0.1,\ \kappa=0$ \\
DAEDL          & 50  & 64  & $1 \times 10^{-3}$ & Adam & Reg $= 0.05$, Dropout $= 0.5$ \\
PostNet        & 50  & 64  & $5 \times 10^{-5}$ & Adam & Latent dim $= 6$, Radial flow \\
MC Dropout     & 50  & 128 & $1 \times 10^{-3}$ & Adam & Dropout $= 0.5$, 50 MC samples \\
Deep Ensemble  & 50  & 128 & $1 \times 10^{-3}$ & Adam & 5 members \\ \hline
\end{tabular}
}
\end{table}

\begin{table}[htbp]
\centering
\caption{Baseline Training Hyperparameters (CIFAR-10)}
\label{tab:cifar_baselines}
\resizebox{\textwidth}{!}{
\begin{tabular}{lccccc}
\hline
\textbf{Method} & \textbf{Epochs} & \textbf{Batch Size} & \textbf{LR} & \textbf{Optimizer} & \textbf{Specific Parameters} \\ \hline
EDL            & 100 & 128 & $0.1$              & SGD (cosine) & Weight decay $= 5\times10^{-4}$, KL annealing $= 10$ epochs \\
R-EDL          & 200 & 64  & $1 \times 10^{-4}$ & Adam         & $\lambda_1=1.0,\ \lambda_2=0.1$ \\
Re-EDL         & 200 & 64  & $1 \times 10^{-4}$ & Adam         & $\lambda_1=1.0,\ \lambda_2=0.8,\ \kappa=0$ \\
DAEDL          & 100 & 64  & $1 \times 10^{-3}$ & Adam         & Reg $= 0.05$, Dropout $= 0.5$ \\
PostNet        & 200 & 64  & $5 \times 10^{-4}$ & Adam         & Latent dim $= 6$, Radial flow \\
MC Dropout     & 100 & 128 & $0.1$              & SGD (cosine) & Dropout $= 0.3$, 50 MC samples \\
Deep Ensemble  & 100 & 128 & $0.1$              & SGD (cosine) & 5 members \\ \hline
\end{tabular}
}
\end{table}

\begin{table}[htbp]
\centering
\caption{Baseline Training Hyperparameters (LAMOST)}
\label{tab:lamost_baselines}
\resizebox{\textwidth}{!}{
\begin{tabular}{lccccc}
\hline
\textbf{Method} & \textbf{Epochs} & \textbf{Batch Size} & \textbf{LR} & \textbf{Optimizer} & \textbf{Specific Parameters} \\ \hline
EDL            & 50 & 128 & $1 \times 10^{-3}$ & Adam & Weight decay $= 1\times10^{-4}$, KL annealing $= 10$ epochs \\
R-EDL          & 60 & 64  & $1 \times 10^{-3}$ & Adam & $\lambda_1=1.0,\ \lambda_2=0.1$ \\
Re-EDL         & 60 & 64  & $1 \times 10^{-3}$ & Adam & $\lambda_1=1.0,\ \lambda_2=0.1,\ \kappa=0$ \\
DAEDL          & 50 & 64  & $1 \times 10^{-3}$ & Adam & Reg $= 0.05$, Dropout $= 0.5$ \\
PostNet        & 50 & 64  & $5 \times 10^{-5}$ & Adam & Latent dim $= 6$, Radial flow \\
MC Dropout     & 50 & 128 & $1 \times 10^{-3}$ & Adam & Dropout $= 0.5$, 50 MC samples \\
Deep Ensemble  & 50 & 128 & $1 \times 10^{-3}$ & Adam & 5 members \\ \hline
\end{tabular}
}
\end{table}

\section{On the OOD Brier Score of DIP-EDL}
\label{appdx:on_BS}

We provide visual evidence for the claim in Section~\ref{sec:experiment_real_cifar} that density estimation difficulty explains DIP-EDL's higher OOD Brier Score on CIFAR-10. Figure~\ref{fig:confidence_comparison} shows the predicted confidence distributions for ID and OOD samples on MNIST and CIFAR-10.

\begin{figure*}[htbp!]
    \centering
    % Left side: MNIST
    \begin{subfigure}[b]{0.48\textwidth}
        \centering
        \includegraphics[width=\linewidth]{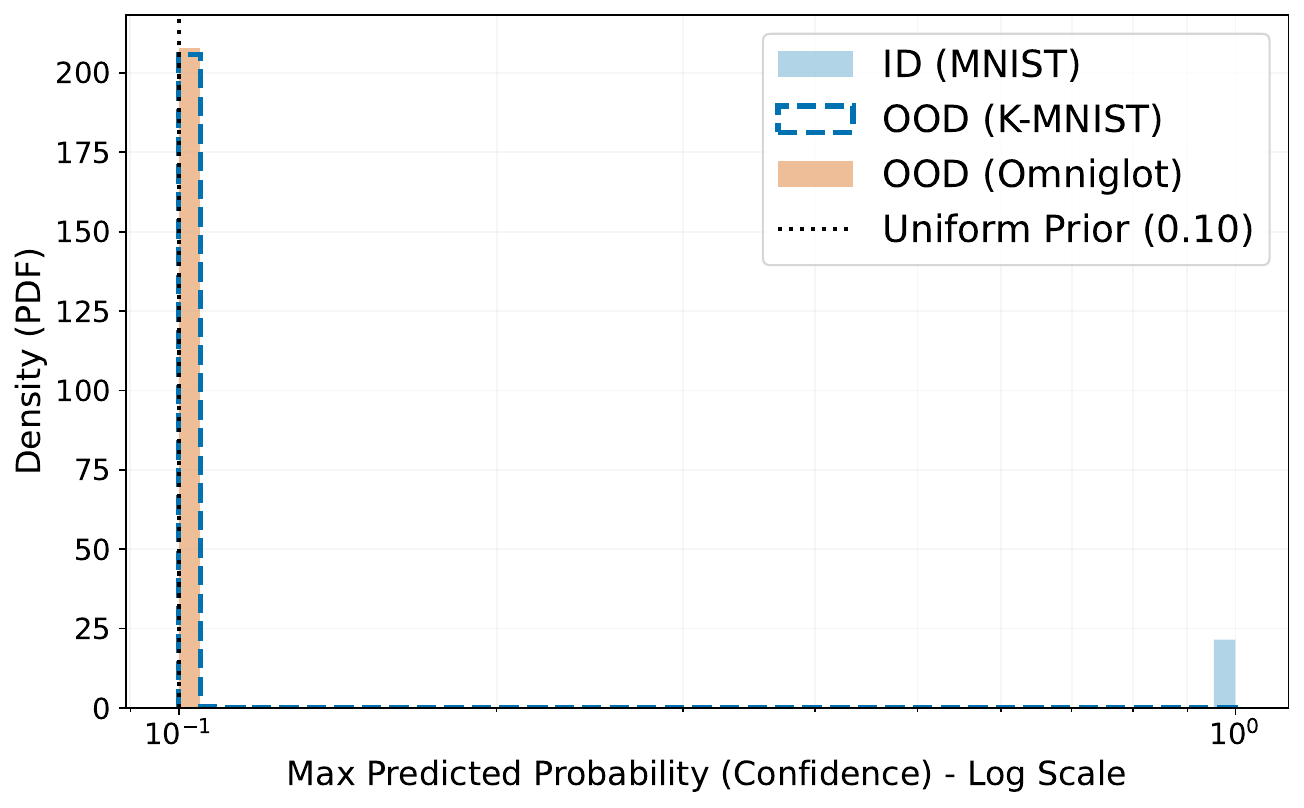}
        \caption{MNIST (ID) vs. KMNIST and Omniglot (OOD)}
        \label{fig:conf_mnist}
    \end{subfigure}
    \hfill
    % Right side: CIFAR-10
    \begin{subfigure}[b]{0.48\textwidth}
        \centering
        \includegraphics[width=\linewidth]{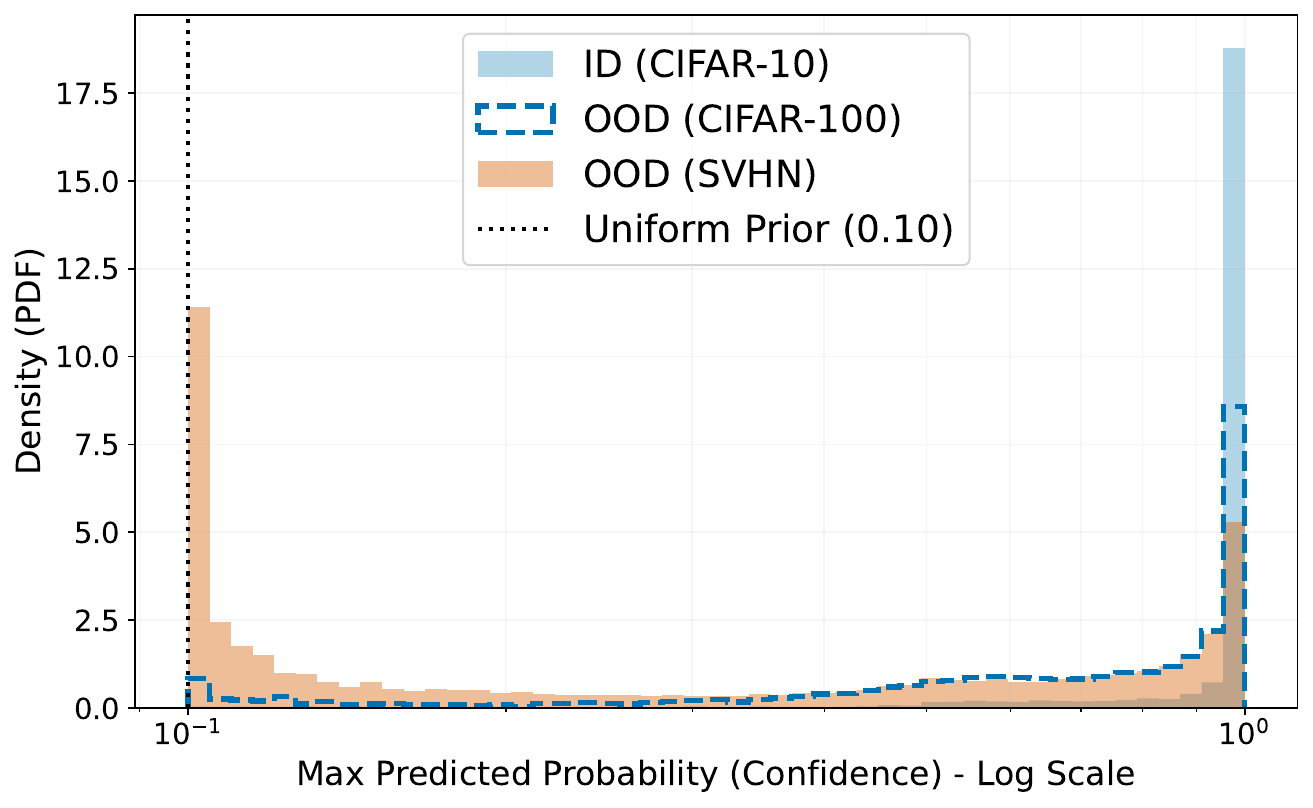}
        \caption{CIFAR-10 (ID) vs. SVHN and CIFAR-100 (OOD)}
        \label{fig:conf_cifar}
    \end{subfigure}
    
    \caption{\textbf{Predicted confidence distributions for DIP-EDL.} Confidence histograms for ID and OOD samples on MNIST (left) and CIFAR-10 (right).}
    \label{fig:confidence_comparison}
\end{figure*}

On MNIST, the density estimator assigns sufficiently low likelihoods to OOD inputs, driving their predicted confidence toward the uniform prior and yielding a near-zero OOD Brier Score. On CIFAR-10, density estimation in the higher-dimensional space is harder: likelihood estimates for OOD inputs remain large enough to sustain high classifier confidence, producing the larger OOD Brier Score discussed in Section~\ref{sec:experiment_real_cifar}, despite strong AUROC and AUPR performance.

\section{Ablation study - Full results}
\label{appdx:ablation_full}

Tables~\ref{tab:ablation_mnist_appdx} and~\ref{tab:ablation_cifar10_appdx} report the complete ablation results on MNIST and CIFAR-10, including all seven component configurations (individual, pairwise, and full model) discussed in Section~\ref{sec:ablation_study}.

\begin{table*}[ht]
    \centering
    \caption{Full results of the ablation study of DIP-EDL.}
    \setlength{\tabcolsep}{3pt} 

    % --- SUB-TABLE 1: MNIST ---
    \begin{subtable}{\textwidth}
        \centering
        \caption{MNIST (ID) vs. K-MNIST and Omniglot (OOD)}
        \resizebox{\textwidth}{!}{
            \begin{tabular}{@{} ccc | cc | cccccc @{}}
                \toprule
                \multicolumn{3}{c|}{\textbf{Components}} & \multicolumn{2}{c|}{\textbf{ID Performance}} & \multicolumn{6}{c}{\textbf{OOD Performance Metrics}} \\ 
                \cmidrule(r){1-3} \cmidrule(lr){4-5} \cmidrule(l){6-11}
    
                % --- SUB-HEADERS ---
                $n$ & $\mathrm{DE}^{\psi}_{X_{i}}$ & $\mathrm{NN}^{\phi}_{X_{i}}$ 
                & \multicolumn{1}{c}{Acc. ($\uparrow$)} & \multicolumn{1}{c|}{BS ($\downarrow$)} 
                & \multicolumn{2}{c}{AUROC ($\uparrow$)} 
                & \multicolumn{2}{c}{AUPR ($\uparrow$)} 
                & \multicolumn{2}{c}{OOD BS ($\downarrow$)} \\ 
                \cmidrule(lr){6-7} \cmidrule(lr){8-9} \cmidrule(lr){10-11}
    
                % --- DATASET ROW ---
                & & & & & \multicolumn{1}{c}{\textbf{K-MNIST}} & \multicolumn{1}{c}{\textbf{Omniglot}} & \multicolumn{1}{c}{\textbf{K-MNIST}} & \multicolumn{1}{c}{\textbf{Omniglot}} & \multicolumn{1}{c}{\textbf{K-MNIST}} & \multicolumn{1}{c}{\textbf{Omniglot}} \\
                \midrule
                
                % --- PAIR COMBINATIONS ---
                \checkmark & \checkmark & $\times$ & $0.0980$ & $0.9000$ & $\mathbf{0.9998}$ & $\mathbf{0.9998}$ & $\mathbf{0.9996}$ & $\mathbf{0.9997}$ & $\mathbf{0.0000}$ & $\mathbf{0.0000}$ \\ % Density + N (2a)
                \checkmark & $\times$ & \checkmark & $\mathbf{0.9958}$ & $\mathbf{0.0069}$ & $0.5138$ & $0.5299$ & $0.5742$ & $0.6408$ & $0.6826$ & $0.7124$ \\ % CNN + N (2b)
                $\times$ & \checkmark & \checkmark & $0.9952$ & $0.7202$ & $0.9996$ & $0.9996$ & $0.9992$ & $0.9994$ & $\mathbf{0.0000}$ & $\mathbf{0.0000}$ \\ % CNN + Density (2c)
    
                \midrule
    
                % --- FULL MODEL ---
                \checkmark & \checkmark & \checkmark & $0.9955$ & $0.0079$ & $\mathbf{0.9998}$ & $\mathbf{0.9998}$ & $0.9995$ & $\mathbf{0.9997}$ & $0.0014$ & $\mathbf{0.0000}$ \\ % Full (3)
    
                \midrule
                
                % --- SINGLE COMPONENTS ---
                \checkmark & $\times$ & $\times$ & $0.0980$ & $0.9000$ & $0.5000$ & $0.5000$ & $0.5000$ & $0.5686$ & $\mathbf{0.0000}$ & $\mathbf{0.0000}$ \\ % Only n (1a)
                $\times$ & \checkmark & $\times$ & $0.0980$ & $0.9000$ & $0.9997$ & $0.9997$ & $0.9993$ & $0.9995$ & $\mathbf{0.0000}$ & $\mathbf{0.0000}$ \\ % Only Density (1b)
                $\times$ & $\times$ & \checkmark & $\mathbf{0.9958}$ & $0.7446$ & $0.5076$ & $0.5174$ & $0.5386$ & $0.6021$ & $0.0056$ & $0.0059$ \\ % Only CNN (1c)
    
                \bottomrule
                \addlinespace
            \end{tabular}
        }
        \label{tab:ablation_mnist_appdx}
    \end{subtable}

    \vspace{1.5em}

    % --- SUB-TABLE 2: CIFAR-10 ---
    \begin{subtable}{\textwidth}
        \centering
        \caption{CIFAR-10 (ID) vs. CIFAR-100 and SVHN (OOD)}
        \resizebox{\textwidth}{!}{
            \begin{tabular}{@{} ccc | cc | cccccc @{}}
                \toprule
                \multicolumn{3}{c|}{\textbf{Components}} & \multicolumn{2}{c|}{\textbf{ID Performance}} & \multicolumn{6}{c}{\textbf{OOD Performance Metrics}} \\ 
                \cmidrule(r){1-3} \cmidrule(lr){4-5} \cmidrule(l){6-11}
    
                % --- SUB-HEADERS ---
                $n$ & $\mathrm{DE}^{\psi}_{X_{i}}$ & $\mathrm{NN}^{\phi}_{X_{i}}$ 
                & \multicolumn{1}{c}{Acc. ($\uparrow$)} & \multicolumn{1}{c|}{BS ($\downarrow$)} 
                & \multicolumn{2}{c}{AUROC ($\uparrow$)} 
                & \multicolumn{2}{c}{AUPR ($\uparrow$)} 
                & \multicolumn{2}{c}{OOD BS ($\downarrow$)} \\ 
                \cmidrule(lr){6-7} \cmidrule(lr){8-9} \cmidrule(lr){10-11}
    
                % --- DATASET ROW ---
                & & & & & \multicolumn{1}{c}{\textbf{CIFAR-100}} & \multicolumn{1}{c}{\textbf{SVHN}} & \multicolumn{1}{c}{\textbf{CIFAR-100}} & \multicolumn{1}{c}{\textbf{SVHN}} & \multicolumn{1}{c}{\textbf{CIFAR-100}} & \multicolumn{1}{c}{\textbf{SVHN}} \\
                \midrule
                
                % --- PAIR COMBINATIONS ---
                \checkmark & \checkmark & $\times$ & $0.1000$ & $0.9000$ & $\mathbf{0.9000}$ & $\mathbf{0.9686}$ & $\mathbf{0.8864}$ & $\mathbf{0.9850}$ & $\mathbf{0.0000}$ & $\mathbf{0.0000}$ \\ % Density + N (2a)
                \checkmark & $\times$ & \checkmark & $\mathbf{0.9496}$ & $\mathbf{0.0831}$ & $0.5036$ & $0.5014$ & $0.5014$ & $0.7224$ & $0.6945$ & $0.6412$ \\ % CNN + N (2b)
                $\times$ & \checkmark & \checkmark & $\mathbf{0.9496}$ & $0.7653$ & $\mathbf{0.9000}$ & $\mathbf{0.9686}$ & $\mathbf{0.8864}$ & $\mathbf{0.9850}$ & $0.0004$ & $\mathbf{0.0000}$ \\ % CNN + Density (2c)
    
                \midrule
    
                % --- FULL MODEL ---
                \checkmark & \checkmark & \checkmark & $\mathbf{0.9496}$ & $0.0978$ & $\mathbf{0.9000}$ & $\mathbf{0.9686}$ & $\mathbf{0.8864}$ & $\mathbf{0.9850}$ & $0.3303$ & $0.1032$ \\ % Full (3)
    
                \midrule
                
                % --- SINGLE COMPONENTS ---
                \checkmark & $\times$ & $\times$ & $0.1000$ & $0.9000$ & $0.5000$ & $0.5000$ & $0.5000$ & $0.7225$ & $\mathbf{0.0000}$ & $\mathbf{0.0000}$ \\ % Only n (1a)
                $\times$ & \checkmark & $\times$ & $0.1000$ & $0.9000$ & $\mathbf{0.9000}$ & $\mathbf{0.9686}$ & $\mathbf{0.8864}$ & $\mathbf{0.9850}$ & $\mathbf{0.0000}$ & $\mathbf{0.0000}$ \\ % Only Density (1b)
                $\times$ & $\times$ & \checkmark & $\mathbf{0.9496}$ & $0.7536$ & $0.5015$ & $0.4991$ & $0.5010$ & $0.7222$ & $0.0057$ & $0.0053$ \\ % Only CNN (1c)
    
                \bottomrule
                \addlinespace
            \end{tabular}
        } 
        \label{tab:ablation_cifar10_appdx}
    \end{subtable}
    
    \label{tab:ablation_combined_appdx}
\end{table*}

\section{Toy Example: Vacuity Under Perfect Interpolation}
\label{appdx:toy_example}

To empirically validate Proposition~\ref{prop:vacuity_perfect_interpolation} (that in-sample vacuity collapses to $\tfrac{K}{\alpha_0 + \nu}$ under perfect interpolation), we train a high-capacity neural network on a synthetic binary classification task designed to reach near-zero training error. We sample $100{,}000$ observations from a balanced mixture of $\mathcal{N}([0,0],I)$ and $\mathcal{N}([7,7],I)$; the two components are well-separated, making near-perfect interpolation achievable by a sufficiently expressive model. We set the prior to $\alpha = [1,1]^\top$ (so $\alpha_0 = 2$). The network uses 3 hidden layers with 64 ReLU units and a ReLU output layer to produce non-negative evidence values. It is trained with the EDL loss for 3000 iterations using Adam, and we repeat training across $\nu \in \{1, 5, 10, 50, 100, 500\}$ with a fixed random seed for reproducibility.

\begin{table}[ht]
\centering
\caption{\textbf{Empirical vacuity under perfect interpolation.} In-sample vacuity statistics for different values of $\nu$, compared to the theoretical prediction $\tfrac{K}{\alpha_0 + \nu}$ from Proposition~\ref{prop:vacuity_perfect_interpolation}.}
\label{tab:toy_vacuity}
\begin{tabular}{cccc}
\toprule
$\nu$ & $\tfrac{K}{\alpha_0 + \nu}$ & Mean & Std.\ Dev. \\
\midrule
1   & 0.667 & 0.667 & 0.0004 \\
5   & 0.286 & 0.287 & 0.0021 \\
10  & 0.167 & 0.167 & 0.0007 \\
50  & 0.039 & 0.038 & 0.0002 \\
100 & 0.020 & 0.020 & 0.0002 \\
500 & 0.004 & 0.004 & 0.0005 \\
\bottomrule
\end{tabular}
\end{table}

Table~\ref{tab:toy_vacuity} shows that empirical vacuity tracks $\tfrac{K}{\alpha_0 + \nu}$ with high precision across all tested values of $\nu$. As $\nu$ increases, the theoretical value $\tfrac{K}{\alpha_0 + \nu}$ decreases, reflecting a more informative prior; the empirical mean and median match this value closely in every case. The sharp concentration (near-zero standard deviation) confirms that, as the model approaches perfect interpolation, in-sample uncertainty is entirely determined by the arbitrary hyperparameter $\nu$, irrespective of the data distribution. This confirms that standard EDL cannot distinguish epistemic from aleatoric uncertainty, and directly motivates the density-informed reparameterization in DIP-EDL.

\section{Density Robustness: Likelihood Scaling and Corruption}
\label{appdx:density_robustness}

We run two experiments to assess the sensitivity of DIP-EDL with respect to the density estimator. In both cases, we verify that DIP-EDL remains robust to imprecision in the density estimate.

\subsection{Likelihood Scaling (\texorpdfstring{$\gamma$}{γ}-sensitivity)}

We introduce a scaling parameter $\gamma > 0$ so that concentration parameters become $\alpha = \mathbf{1} + \gamma \cdot n \cdot \mathrm{DE}^{\psi}_{X} \cdot \mathrm{NN}^{\phi}_{X}$, where $\gamma = 1$ corresponds to standard DIP-EDL. Tables~\ref{tab:nu_edl_gamma_mnist_appdx} and~\ref{tab:nu_edl_gamma_cifar_appdx} show results across two orders of magnitude of $\gamma$ on MNIST and CIFAR-10. ID accuracy, AUROC, and AUPR remain essentially unchanged, demonstrating robustness to the overall scale of the density estimate. The only metric that varies with $\gamma$ is OOD Brier Score, which increases monotonically as larger $\gamma$ amplifies model confidence on OOD inputs, consistent with our discussion in Section~\ref{sec:experiment_real_cifar}.

\begin{table*}[ht]
    \centering
    \caption{\textbf{Likelihood scaling ($\gamma$) on MNIST.} Effect of the likelihood scaling sensitivity parameter $\gamma$ on ID accuracy, ID BS, AUROC, AUPR, and OOD BS, evaluated on MNIST.}
    \setlength{\tabcolsep}{3pt}
    \resizebox{\textwidth}{!}{
        \begin{tabular}{@{} c | c c | c c c c c c @{}}
            \toprule
            \multirow{3}{*}{\textbf{$\gamma$}} & \multicolumn{2}{c|}{\textbf{ID Performance}} & \multicolumn{6}{c}{\textbf{OOD Performance Metrics}} \\
            \cmidrule(lr){2-3} \cmidrule(lr){4-9}
            & \multicolumn{1}{c}{Acc. ($\uparrow$)} & \multicolumn{1}{c|}{BS ($\downarrow$)}
            & \multicolumn{2}{c}{AUROC ($\uparrow$)}
            & \multicolumn{2}{c}{AUPR ($\uparrow$)}
            & \multicolumn{2}{c}{OOD BS ($\downarrow$)} \\
            \cmidrule(lr){4-5} \cmidrule(lr){6-7} \cmidrule(lr){8-9}
            & \multicolumn{2}{c|}{\textbf{MNIST}} & \multicolumn{1}{c}{\textbf{K-MNIST}} & \multicolumn{1}{c}{\textbf{Omniglot}} & \multicolumn{1}{c}{\textbf{K-MNIST}} & \multicolumn{1}{c}{\textbf{Omniglot}} & \multicolumn{1}{c}{\textbf{K-MNIST}} & \multicolumn{1}{c}{\textbf{Omniglot}} \\
            \midrule
            \textbf{0.1}  & $0.9951$ & $0.0092$ & $0.9997$ & $0.9998$ & $0.9995$ & $0.9996$ & $0.0005$ & $0.0000$ \\
            \textbf{0.5}  & $0.9951$ & $0.0090$ & $0.9997$ & $0.9998$ & $0.9995$ & $0.9996$ & $0.0012$ & $0.0000$ \\
            \textbf{1.0} {\small(DIP-EDL)}  & $0.9952$ & $0.0088$ & $0.9998$ & $0.9999$ & $0.9997$ & $0.9998$ & $0.0017$ & $0.0000$ \\
            \textbf{2.0}  & $0.9951$ & $0.0089$ & $0.9997$ & $0.9998$ & $0.9995$ & $0.9996$ & $0.0021$ & $0.0000$ \\
            \textbf{5.0}  & $0.9954$ & $0.0087$ & $0.9999$ & $0.9999$ & $0.9998$ & $0.9999$ & $0.0030$ & $0.0000$ \\
            \textbf{10.0} & $0.9953$ & $0.0088$ & $0.9998$ & $0.9999$ & $0.9997$ & $0.9998$ & $0.0034$ & $0.0000$ \\
            \bottomrule
        \end{tabular}
    }
    \label{tab:nu_edl_gamma_mnist_appdx}
\end{table*}

\begin{table*}[ht]
    \centering
    \caption{\textbf{Likelihood scaling ($\gamma$) on CIFAR-10.} Effect of the likelihood scaling sensitivity parameter $\gamma$ on ID accuracy, ID BS, AUROC, AUPR, and OOD BS, evaluated on CIFAR-10.}
    \setlength{\tabcolsep}{3pt}
    \resizebox{\textwidth}{!}{
        \begin{tabular}{@{} c | c c | c c c c c c @{}}
            \toprule
            \multirow{3}{*}{\textbf{$\gamma$}} & \multicolumn{2}{c|}{\textbf{ID Performance}} & \multicolumn{6}{c}{\textbf{OOD Performance Metrics}} \\
            \cmidrule(lr){2-3} \cmidrule(lr){4-9}
            & \multicolumn{1}{c}{Acc.\ ($\uparrow$)} & \multicolumn{1}{c|}{BS ($\downarrow$)}
            & \multicolumn{2}{c}{AUROC ($\uparrow$)}
            & \multicolumn{2}{c}{AUPR ($\uparrow$)}
            & \multicolumn{2}{c}{OOD BS ($\downarrow$)} \\
            \cmidrule(lr){4-5} \cmidrule(lr){6-7} \cmidrule(lr){8-9}
            & \multicolumn{2}{c|}{\textbf{CIFAR-10}} & \multicolumn{1}{c}{\textbf{CIFAR-100}} & \multicolumn{1}{c}{\textbf{SVHN}} & \multicolumn{1}{c}{\textbf{CIFAR-100}} & \multicolumn{1}{c}{\textbf{SVHN}} & \multicolumn{1}{c}{\textbf{CIFAR-100}} & \multicolumn{1}{c}{\textbf{SVHN}} \\
            \midrule
    \textbf{0.1} & $0.9480$ & $0.1148$ & $0.9010$ & $0.9630$ & $0.8860$ & $0.9819$ & $0.2132$ & $0.0604$ \\
    % \textbf{0.2} & $0.9480$ & $0.1087$ & $0.9010$ & $0.9630$ & $0.8860$ & $0.9819$ & $0.2455$ & $0.0764$ \\
    \textbf{0.5} & $0.9480$ & $0.1029$ & $0.9010$ & $0.9630$ & $0.8860$ & $0.9819$ & $0.2885$ & $0.1009$ \\
    \textbf{1.0} {\small(DIP-EDL)} & $0.9480$ & $0.0996$ & $0.9010$ & $0.9630$ & $0.8860$ & $0.9819$ & $0.3206$ & $0.1217$ \\
    \textbf{2.0} & $0.9480$ & $0.0969$ & $0.9010$ & $0.9630$ & $0.8860$ & $0.9819$ & $0.3518$ & $0.1439$ \\
    \textbf{5.0} & $0.9480$ & $0.0940$ & $0.9010$ & $0.9630$ & $0.8860$ & $0.9819$ & $0.3911$ & $0.1747$ \\
    \textbf{10.0} & $0.9480$ & $0.0920$ & $0.9010$ & $0.9630$ & $0.8860$ & $0.9819$ & $0.4191$ & $0.1988$ \\
            \bottomrule
        \end{tabular}
    }
    \label{tab:nu_edl_gamma_cifar_appdx}
\end{table*}

\subsection{Density Corruption (\texorpdfstring{$\sigma$}{σ}-robustness)}

We simulate degraded density estimates by injecting Gaussian noise $\mathcal{N}(0, \sigma^2)$ into the z-scored log-probability before computing $\mathrm{DE}^{\psi}_X$. The case $\sigma = 0$ corresponds to the standard clean run, and $\sigma = \infty$ replaces the density entirely with i.i.d.\ $\mathcal{N}(0,1)$ noise (a fully uninformative density estimator). Tables~\ref{tab:density_corruption_mnist_appdx} and~\ref{tab:density_corruption_cifar_appdx} show that DIP-EDL is very robust: all metrics remain essentially unchanged up to $\sigma = 2$, with graceful degradation at $\sigma = 5$. OOD detection becomes near-random only when the density estimator is fully uninformative ($\sigma = \infty$), consistent with the ablation results in Section~\ref{sec:ablation_study}.

\begin{table*}[ht]
    \centering
    \caption{\textbf{Density corruption ($\sigma$) on MNIST.} Effect of Gaussian density corruption sensitivity parameter $\sigma$ on ID accuracy, ID BS, AUROC, AUPR, and OOD BS, evaluated on MNIST.}
    \setlength{\tabcolsep}{3pt}
    \resizebox{\textwidth}{!}{
        \begin{tabular}{@{} c | c c | c c c c c c @{}}
            \toprule
            \multirow{3}{*}{\textbf{$\sigma$}} & \multicolumn{2}{c|}{\textbf{ID Performance}} & \multicolumn{6}{c}{\textbf{OOD Performance Metrics}} \\
            \cmidrule(lr){2-3} \cmidrule(lr){4-9}
            & \multicolumn{1}{c}{Acc. ($\uparrow$)} & \multicolumn{1}{c|}{BS ($\downarrow$)}
            & \multicolumn{2}{c}{AUROC ($\uparrow$)}
            & \multicolumn{2}{c}{AUPR ($\uparrow$)}
            & \multicolumn{2}{c}{OOD BS ($\downarrow$)} \\
            \cmidrule(lr){4-5} \cmidrule(lr){6-7} \cmidrule(lr){8-9}
            & \multicolumn{2}{c|}{\textbf{MNIST}} & \multicolumn{1}{c}{\textbf{K-MNIST}} & \multicolumn{1}{c}{\textbf{Omniglot}} & \multicolumn{1}{c}{\textbf{K-MNIST}} & \multicolumn{1}{c}{\textbf{Omniglot}} & \multicolumn{1}{c}{\textbf{K-MNIST}} & \multicolumn{1}{c}{\textbf{Omniglot}} \\
            \midrule
            \textbf{$0$} {\small(clean)}  & $0.9951$ & $0.0089$ & $0.9998$ & $0.9998$ & $0.9996$ & $0.9997$ & $0.0017$ & $0.0000$ \\
            \textbf{$0.5$}        & $0.9952 \pm 0.0001$ & $0.0089 \pm 0.0001$ & $0.9998 \pm 0.0001$ & $0.9998 \pm 0.0001$ & $0.9996 \pm 0.0001$ & $0.9997 \pm 0.0001$ & $0.0017 \pm 0.0000$ & $0.0000 \pm 0.0000$ \\
            \textbf{$1$}          & $0.9952 \pm 0.0001$ & $0.0089 \pm 0.0001$ & $0.9998 \pm 0.0001$ & $0.9998 \pm 0.0001$ & $0.9996 \pm 0.0001$ & $0.9997 \pm 0.0001$ & $0.0017 \pm 0.0001$ & $0.0000 \pm 0.0000$ \\
            \textbf{$2$}          & $0.9952 \pm 0.0001$ & $0.0093 \pm 0.0002$ & $0.9997 \pm 0.0001$ & $0.9998 \pm 0.0001$ & $0.9995 \pm 0.0001$ & $0.9997 \pm 0.0001$ & $0.0019 \pm 0.0001$ & $0.0000 \pm 0.0000$ \\
            \textbf{$5$}          & $0.9952 \pm 0.0001$ & $0.0433 \pm 0.0008$ & $0.9984 \pm 0.0001$ & $0.9998 \pm 0.0001$ & $0.9985 \pm 0.0001$ & $0.9997 \pm 0.0001$ & $0.0035 \pm 0.0002$ & $0.0000 \pm 0.0000$ \\
            \midrule
            \textbf{$\infty$}     & $0.9955 \pm 0.0000$ & $0.0078 \pm 0.0000$ & $0.5016 \pm 0.0031$ & $0.5023 \pm 0.0035$ & $0.5010 \pm 0.0026$ & $0.5707 \pm 0.0027$ & $0.6957 \pm 0.0000$ & $0.7443 \pm 0.0000$ \\
            \bottomrule
        \end{tabular}
    }
    \label{tab:density_corruption_mnist_appdx}
\end{table*}

\begin{table}[H]
    \centering
    \caption{\textbf{Density corruption ($\sigma$) on CIFAR-10.} Effect of Gaussian density corruption sensitivity parameter $\sigma$ on ID accuracy, ID BS, AUROC, AUPR, and OOD BS, evaluated on CIFAR-10.}
    \setlength{\tabcolsep}{3pt}
    \resizebox{\textwidth}{!}{
        \begin{tabular}{@{} c | c c | c c c c c c @{}}
            \toprule
            \multirow{3}{*}{\textbf{$\sigma$}} & \multicolumn{2}{c|}{\textbf{ID Performance}} & \multicolumn{6}{c}{\textbf{OOD Performance Metrics}} \\
            \cmidrule(lr){2-3} \cmidrule(lr){4-9}
            & \multicolumn{1}{c}{Acc.\ ($\uparrow$)} & \multicolumn{1}{c|}{BS ($\downarrow$)}
            & \multicolumn{2}{c}{AUROC ($\uparrow$)}
            & \multicolumn{2}{c}{AUPR ($\uparrow$)}
            & \multicolumn{2}{c}{OOD BS ($\downarrow$)} \\
            \cmidrule(lr){4-5} \cmidrule(lr){6-7} \cmidrule(lr){8-9}
            & \multicolumn{2}{c|}{\textbf{CIFAR-10}} & \multicolumn{1}{c}{\textbf{CIFAR-100}} & \multicolumn{1}{c}{\textbf{SVHN}} & \multicolumn{1}{c}{\textbf{CIFAR-100}} & \multicolumn{1}{c}{\textbf{SVHN}} & \multicolumn{1}{c}{\textbf{CIFAR-100}} & \multicolumn{1}{c}{\textbf{SVHN}} \\
            \midrule
    \textbf{$0$} {\small(clean)} & $0.9480$ & $0.0996$ & $0.9010$ & $0.9630$ & $0.8860$ & $0.9819$ & $0.3206$ & $0.1217$ \\
    \textbf{$0.5$} & $0.9480$ & $0.0996 \pm 0.0001$ & $0.8988 \pm 0.0006$ & $0.9623 \pm 0.0001$ & $0.8848 \pm 0.0003$ & $0.9817 \pm 0.0000$ & $0.3202 \pm 0.0005$ & $0.1220 \pm 0.0002$ \\
    \textbf{$1$} & $0.9480$ & $0.1000 \pm 0.0002$ & $0.8925 \pm 0.0010$ & $0.9601 \pm 0.0003$ & $0.8811 \pm 0.0006$ & $0.9809 \pm 0.0001$ & $0.3193 \pm 0.0009$ & $0.1231 \pm 0.0002$ \\
    \textbf{$2$} & $0.9480$ & $0.1023 \pm 0.0003$ & $0.8734 \pm 0.0015$ & $0.9514 \pm 0.0006$ & $0.8682 \pm 0.0010$ & $0.9778 \pm 0.0002$ & $0.3167 \pm 0.0015$ & $0.1274 \pm 0.0003$ \\
    \textbf{$5$} & $0.9480$ & $0.1451 \pm 0.0011$ & $0.8002 \pm 0.0027$ & $0.8995 \pm 0.0017$ & $0.8081 \pm 0.0021$ & $0.9574 \pm 0.0006$ & $0.3099 \pm 0.0010$ & $0.1505 \pm 0.0008$ \\
    \textbf{$\infty$} & $0.9480$ & $0.0851 \pm 0.0000$ & $0.5031 \pm 0.0040$ & $0.5016 \pm 0.0056$ & $0.5016 \pm 0.0046$ & $0.7226 \pm 0.0039$ & $0.7004 \pm 0.0000$ & $0.6514 \pm 0.0000$ \\
            \bottomrule
        \end{tabular}
    }
    \label{tab:density_corruption_cifar_appdx}
\end{table}

\end{document}